\documentclass{article}

\PassOptionsToPackage{authoryear,compress}{natbib}
\usepackage[dblblindworkshop,final]{neurips_2026}
\workshoptitle{Diffusion Language Models: Foundations, Efficiency, and Reasoning}

\usepackage[utf8]{inputenc}
\usepackage[T1]{fontenc}
\usepackage{hyperref}
\usepackage{url}
\usepackage{booktabs}
\usepackage{array}
\usepackage{tabularx}
\usepackage{graphicx}
\usepackage{adjustbox}
\usepackage{amsmath,amssymb,amsfonts,mathtools}
\usepackage{bm}
\usepackage{amsthm}
\usepackage{microtype}
\usepackage{xcolor}
\usepackage{enumitem}
\usepackage{placeins}
\usepackage{float}
\usepackage{needspace}
\usepackage{algorithm}
\usepackage{algpseudocode}
\usepackage{tikz}
\usetikzlibrary{arrows.meta}

\hypersetup{
  hypertexnames=false,
  colorlinks=true,
  linkcolor=blue!45!black,
  citecolor=blue!45!black,
  urlcolor=blue!45!black,
  pdftitle={From Truncation to Commitment: Persistent Context in Uniform Discrete Diffusion},
  pdfauthor={Satoshi Hayakawa}
}

\newcommand{\KL}{D_{\mathrm{KL}}}
\newcommand{\kldiv}{\,\Vert\,}

\newcommand{\TC}{\operatorname{TC}}
\newcommand{\Ent}{\operatorname{Ent}}
\newcommand{\E}{\mathbb E}
\newcommand{\1}{\mathbf 1}
\newcommand{\CRS}{\textsc{CRS}}

\newtheorem{theorem}{Theorem}
\newtheorem{proposition}[theorem]{Proposition}
\newtheorem{corollary}[theorem]{Corollary}

\theoremstyle{definition}

\newcolumntype{Y}{>{\raggedright\arraybackslash}X}

\title{From Truncation to Commitment:\\
Persistent Context in Uniform Discrete Diffusion}

\author{Satoshi Hayakawa \\
The University of Tokyo \\
\texttt{hayakawa@mist.i.u-tokyo.ac.jp}}

\begin{document}
\raggedbottom

\maketitle
\begin{abstract}
Uniform-state discrete diffusion models update all tokens in parallel while
keeping every position revisable. Even when the commonly used top-$p$
rule leaves only one candidate at a position, that choice affects
only the current reverse step and can be revised at the next sampling step. We ask
what changes when selected hypotheses instead become persistent context for
later predictions. We therefore propose \emph{committed reveal sampling}
(\CRS), a training-free sampler that stores selected
argmax tokens and inserts them into subsequent model inputs.
Our analysis separates persistent context from truncation. Within one reverse
step, adaptive top-$p$, randomized or not, reproduces the update induced by a
visible token only when a likelihood-ratio ordering holds, while in a simple
latent-mode model, keeping the selected token visible helps later parallel
predictions agree on the same sequence-level choice.
Empirically, paired experiments on Duo-distilled then separate
this persistent effect from single-step top-$p$ restriction and scalar
temperature scaling. Under the same finalization rule, \CRS{} without top-$p$
truncation reaches lower generative perplexity (GenPPL) than fixed $p=0.95$ and
$p=0.9$ baselines across budgets of 8--64 function evaluations (NFE). At 64 NFE, the
comparison at matched unigram entropy also gives lower GenPPL for \CRS{},
shifting the GenPPL--entropy tradeoff under this metric. Base Duo shows the same
direction in a descriptive comparison, while other diversity and continuation
metrics can rank these operating points differently. A distributional distance to held-out text (SOL) under a fixed GPT-2-Large representation also decreases with persistent context at the same late top-$p$ threshold. These results identify support
restriction and persistent context as distinct controls of that tradeoff.
\end{abstract}

\section{Introduction}
\label{sec:intro}

Discrete diffusion models generate a sequence through repeated parallel token
updates~\citep{austin2021structured,campbell2022continuous,lou2024discrete,
sahoo2024simple}. Masked models gradually replace an absorbing mask with visible
tokens, and each reveal supplies context to later updates and fixes part of the
generation order~\citep{chang2022maskgit,hayakawa2026demystifying}. In contrast,
uniform-state models such as Duo use vocabulary symbols as noisy states and
keep every position revisable~\citep{sahoo2025duo}. Such revisability supports
continued sequence-wide revision, but removes the explicit record of earlier
decisions.

In practice, uniform samplers often restrict their predictions over clean
tokens before each reverse update. The released Duo sampler uses top-$p$, or
nucleus, sampling to restrict its clean-token predictions before each reverse
update~\citep{holtzman2020curious,deschenaux2026psi}. Late in the reverse
process, the retained support is often a single token, but even then that token
is used for the current update only.

Figure~\ref{fig:concept} contrasts the two samplers. Both evaluate the same frozen denoiser once per step and re-sample every position. Top-$p$ restricts the current clean-token law and influences later predictions through the updated state. \emph{Committed reveal sampling} (\CRS) retains selected labels in the model input. After a warmup with the native sampler, it selects positions with concentrated predictions, stores their argmax tokens, and inserts them into subsequent model inputs. The mutable state is never clamped: all positions use the native reverse-transition formula, with probabilities computed from the label-overwritten input and conditioned on the original state.

\begin{figure}[H]
  \centering
  \begin{tikzpicture}[x=1cm,y=.92cm,
  tok/.style={draw=black!55,rounded corners=1.5pt,minimum width=.6cm,minimum height=.44cm,inner sep=1pt,font=\footnotesize,fill=white},
  stored/.style={tok,draw=green!45!black,fill=green!8,line width=.7pt},
  sing/.style={tok,draw=orange!85!black,fill=orange!8,line width=.7pt},
  blk/.style={draw=blue!60!black,fill=blue!6,rounded corners=2pt,inner sep=3pt,font=\footnotesize,align=center},
  box/.style={draw=black!55,rounded corners=2pt,inner sep=3pt,font=\footnotesize,align=center,fill=white},
  lab/.style={font=\footnotesize\bfseries,anchor=east},
  sym/.style={font=\footnotesize,anchor=east},
  note/.style={font=\scriptsize,text=black!65,align=center},
  arr/.style={-{Latex[length=1.5mm,width=1.1mm]},black!60,line width=.5pt},
  garr/.style={-{Latex[length=1.5mm,width=1.1mm]},green!45!black,line width=.6pt}]
\node[anchor=west,font=\footnotesize\bfseries] at (-.05,4.0) {(a) One sampling step};
\node[lab] at (.95,3.1) {Input};  \node[sym] at (.95,2.72) {$\widetilde{\bm x}_t$};
\node[lab] at (.95,1.95) {State}; \node[sym] at (.95,1.57) {$\bm x_t$};
\node[lab] at (.95,.8)  {Next state}; \node[sym] at (.95,.42) {$\bm x_s$};
\foreach \i/\w in {0/the,2/sat} {\node[tok] at (1.3+\i*.68,2.9) {\w}; \node[tok] at (1.3+\i*.68,1.75) {\w}; \node[tok] at (1.3+\i*.68,.6) {\w};}
\node[stored] at (1.98,2.9) {cat};
\node[tok] at (1.98,1.75) {cot};
\node[tok] at (1.98,.6) {cut};
\node[note,text=green!45!black] at (1.98,3.35) {stored label (\CRS)};
\node[note] at (1.98,2.33) {$\neq$};
\node[blk,minimum width=2.4cm,minimum height=.8cm] (den) at (4.95,2.9) {Denoiser\\[-1pt]+ optional top-$p$};
\node[box,minimum width=2.4cm,minimum height=.48cm] (rev) at (4.95,1.75) {Reverse transition};
\draw[arr] (2.98,2.9) -- (den.west);
\draw[arr] (den.south) -- (rev.north) node[midway,note,anchor=west,xshift=1pt] {clean-token law};
\draw[arr] (2.98,1.75) -- (rev.west);
\draw[arr] (rev.south) -- (4.95,.6) -- (2.98,.6);
\node[note,anchor=west,align=left] at (5.0,1.1) {conditions on\\[-2pt]the state};
\node[note,anchor=west,align=left] at (-.05,-.05) {Every position is re-sampled at every step.};
\begin{scope}[xshift=6.95cm]
\node[anchor=west,font=\footnotesize\bfseries] at (-.05,4.0) {(b) One position over sampling steps};
\node[lab] at (1.3,3.1) {Top-$p$}; \node[note,anchor=east] at (1.3,2.72) {input = state};
\node[lab] at (1.3,1.95) {\CRS{} input}; \node[sym] at (1.3,1.57) {$\widetilde x_t^{\,i}$};
\node[lab] at (1.3,1.15) {\CRS{} state}; \node[sym] at (1.3,.77) {$x_t^i$};
\node[note,anchor=east] at (1.3,.3) {step};
\foreach \i/\w/\s/\c/\cs/\r in {0/cot/tok/cot/tok/{=},1/cut/sing/cut/tok/{=},2/cat/tok/cat/stored/{=},3/cap/tok/cat/stored/{$\neq$},4/cat/tok/cat/stored/{=}} {
  \node[\s] at (1.75+\i*.98,2.9) {\w};
  \node[\cs] at (1.75+\i*.98,1.75) {\c};
  \node[tok] at (1.75+\i*.98,1.0) {\w};
  \node[note] at (1.75+\i*.98,1.37) {\r};
}
\foreach \i/\k in {0/k,1/k+1,2/k+2,3/k+3,4/k+4} {\node[font=\footnotesize] at (1.75+\i*.98,.3) {$\k$};}
\foreach \i in {0,...,3} {\draw[arr] (1.75+\i*.98+.31,2.9) -- (1.75+\i*.98+.67,2.9); \draw[arr] (1.75+\i*.98+.31,1.0) -- (1.75+\i*.98+.67,1.0);}
\node[note,text=orange!85!black,anchor=south] at (2.95,3.2) {prediction collapses to ``cat'':\\[-2pt]used for this update only};
\draw[garr] (2.9,2.05) to[bend left=35] (3.6,2.05);
\node[note,text=green!45!black] at (2.6,2.42) {select argmax};
\draw[green!45!black,line width=.6pt] (3.4,2.2) -- (5.95,2.2);
\node[note,text=green!45!black] at (4.7,2.42) {stays in later inputs};
\node[note,anchor=west,align=left] at (-.05,-.05) {The state may still change; only the model input keeps the label.};
\end{scope}
\end{tikzpicture}

  \caption{Support restriction versus persistent context. (a) One sampling step: the denoiser reads the model input, which equals the state except where \CRS{} has stored a label; top-$p$ restricts only this step's clean-token law, and the reverse transition conditions on the state. (b) One position over steps: a top-$p$ singleton affects one update and the next input is the new state; a \CRS{} argmax stays in later inputs while the state keeps being re-sampled. Tokens are illustrative.}
\label{fig:concept}
\end{figure}

The central intuition is that parallel updates can combine individually plausible tokens from incompatible sequence-level alternatives. A reliable visible token can reduce this ambiguity and steer later predictions toward the same alternative, whereas an incorrect selection can steer them toward the wrong one. This raises two questions: when should a token be selected, and how does keeping it visible affect later predictions?

We study these questions theoretically and empirically. Theoretically, no adaptive top-$p$ rule can reproduce the reverse update induced by a visible token unless a likelihood-ratio ordering holds, and a latent-mode analysis shows how persistent visibility can coordinate later predictions around the selected alternative; a result on the exact forward process in the appendix motivates the warmup. Empirically, paired experiments compare CRS with top-$p$ samplers and measure the resulting tradeoff between generative perplexity (GenPPL) and unigram entropy across compute budgets on the base and distilled Duo checkpoints \citep{sahoo2025duo}.

The paper makes three contributions.
\begin{enumerate}[leftmargin=*,itemsep=2pt,topsep=3pt]
  \item \textbf{Persistence and support restriction.} We distinguish the support restriction induced by top-$p$, which applies only to the current sampling step, from persistent context that remains visible in later steps, and formalize \CRS{} at the uniform reverse interface.
  \item \textbf{Theoretical analysis.} Within one reverse step, adaptive top-$p$, randomized or not, reproduces the update induced by a context-dependent proposal only under a likelihood-ratio ordering (Proposition~\ref{prop:adaptive-reach}), so persistent context is not a reparametrization of truncation. In a simple latent-mode model, persistent visible context helps later positionwise predictions favor the same sequence-level alternative (Theorem~\ref{thm:shared-mode}); under the exact forward process, waiting until the state is less noisy cannot increase the optimal selection error (Proposition~\ref{prop:warmup-blackwell}, Appendix~\ref{app:warmup-theory}).
  \item \textbf{Experiments.} Under a common finalization rule, \CRS{} without
  top-$p$ truncation has lower GenPPL than fixed global
  $p\in\{0.95,0.9\}$ baselines on Duo-distilled across 8--64 NFE. Base Duo
  shows the same ordering in a descriptive comparison. At 64 NFE,
  matched-entropy comparisons also favor \CRS{} on Duo-distilled.
\end{enumerate}

\section{Related Work}
\label{sec:related-work}

\paragraph{Uniform-state discrete diffusion.}
Duo and its leave-one-out (LOO) analysis define the uniform reverse interface used here
\citep{sahoo2025duo,gourevitch2026uniform}.
\mbox{\citet{deschenaux2026psi}} report Duo top-$p$ sweeps with an ancestral endpoint while developing predictor--corrector samplers. Our study asks a different question: we compare support restriction applied only in the current sampling step with persistent context visible in later steps, using the same coordinatewise argmax decoder in the final step for all methods. Because the endpoint protocols differ, the resulting values are not directly comparable.

\paragraph{Joint distribution modeling.}
The native parallel interface constructs each update from a product of coordinatewise predictions; we refer to this interface as a product denoiser. Prior work enriches this interface through correlation-aware mixtures, energy corrections, coupled transitions, or auxiliary simplex-valued reverse states
\citep{hayakawa2025distillation,xu2024edlm,li2026codd,sakurai2026simplax}.
\CRS{} instead keeps the checkpoint fixed and changes the context available across sampling steps by exposing selected clean-token hypotheses to later denoiser inputs.

\paragraph{Masked diffusion and commitment.}
Existing token-commitment methods mainly target masked diffusion.
Masked diffusion turns an absorbing mask into visible clean tokens, so reveal
order is part of the sampler~\citep{sahoo2024simple,shi2024simplified}.
MaskGIT and later methods select or remask positions using confidence or entropy
\citep{chang2022maskgit,hayakawa2026demystifying,benhamu2025entropy,
wang2025remasking}. Fast-dLLM commits confident tokens for parallel decoding
and cache reuse~\citep{wu2025fastdllm}, while deferred commitment postpones
uncertain blockwise decisions~\citep{shu2026deferred}. Other policies use future self-consistency, trajectory stability, or coordination among dependent token values to determine which token predictions to commit and when
\citep{wang2026selfcontained,sun2026pathmatters,wang2026tacg,yao2026cocommit}.
\CRS{} asks a different question: what happens when a stored clean-token hypothesis remains visible across steps in an otherwise fully revisable uniform-state sampler?

\paragraph{Generalized and larger-scale models.}
GIDD generalizes discrete diffusion to interpolating kernels that combine
masking and uniform noise~\citep{vonrutte2025gidd}, and Sumi scales a
pure-uniform instantiation with a time-agnostic bidirectional
Transformer~\citep{ye2026sumi}. Extending persistent-context controls to such
models, and to the masked-diffusion view underlying LLaDA and
Dream~\citep{nie2025llada,ye2025dream}, is a natural next step.

\section{Uniform Reverse Sampling and
\texorpdfstring{Top-$p$ Support Restriction}{Top-p Support Restriction}}
\label{sec:uniform-reverse}

We first derive the native reverse interface and locate top-$p$ within one
reverse step. This identifies which clean-token distribution top-$p$ restricts
and why even a singleton choice is not carried explicitly into the next
step. Section~\ref{sec:persistent-context} then introduces persistent context
at that interface.

\subsection{From the LOO law to the native reverse step}

Let $\mathcal S=[V]=\{1,\ldots,V\}$ be the vocabulary and
$[L]=\{1,\ldots,L\}$ be the set of positions. Let $q_0$ be the data distribution on
$\mathcal S^L$, and let $(\bm X_t)_{0\le t\le T}$ be a coordinatewise Markov forward process whose transition kernels belong to the uniform-corruption family. In particular,
$\bm X_0\sim q_0$ and
$\bm X_t\mid\bm X_0\sim q_{t\mid0}(\cdot\mid\bm X_0)$. For uniform
corruption, this conditional law factorizes as
\begin{equation}
 q_{t\mid0}(\bm x_t\mid\bm x_0)
 =\prod_{i=1}^Lq_{t\mid0}^i(x_t^i\mid x_0^i),
 \qquad
 q_{t\mid0}^i(z\mid x)
 =(1-\alpha_t)\1\{z=x\}+\frac{\alpha_t}{V}.
 \label{eq:uniform-forward-main}
\end{equation}
Here $\alpha_t$ is monotonically increasing,
with $\alpha_t=0$ giving clean data and $\alpha_t=1$ giving uniform noise,
and $\bm x=(x^i)_{i=1}^L$. Let $\bm x^{-i}$ denote the sequence with
coordinate $i$ removed.

For $0\le\alpha_s<\alpha_t\le1$, Markov consistency requires
$q_{t|0}^i(z\mid x) = \sum_y q_{t|s}^i(z\mid y)q_{s|0}^i(y\mid x)$.
Since uniform channels are closed under composition, $q_{t|s}$ is given as
\begin{equation}
 q_{t\mid s}^i(z\mid y)
 =(1-\beta_{t,s})\1\{z=y\}+\frac{\beta_{t,s}}{V},
 \qquad
 \beta_{t,s}=\frac{\alpha_t-\alpha_s}{1-\alpha_s}.
 \label{eq:intermediate-uniform-main}
\end{equation}
The reverse update also needs the earlier noisy token conditioned on both the
clean token and the current noisy token. %
Bayes' rule gives this exact coordinate bridge:
\begin{equation}
 q_{s\mid0,t}^i(y\mid x,z)
 =\frac{q_{s\mid0}^i(y\mid x)q_{t\mid s}^i(z\mid y)}
 {q_{t\mid0}^i(z\mid x)}.
 \label{eq:bridge-main}
\end{equation}

The released Duo model takes the full noisy sequence $\bm x_t$ as input and
outputs, at each coordinate $i$, a normalized categorical law
$h_{\theta,t}^i(\cdot\mid\bm x_t)$ over clean labels
\citep{sahoo2025duo}. Subsequent analysis by
\citet{gourevitch2026uniform} showed that the corresponding population target
is the leave-one-out (LOO) conditional, which predicts the clean token at
coordinate $i$ from all other noisy coordinates:
\begin{equation}
 q_{0\mid t}^i(x\mid\bm x_t^{-i})
 :=\mathbb P(X_0^i=x\mid\bm X_t^{-i}=\bm x_t^{-i})
 =\mathbb E\!\left[
 q_{0\mid t}^i(x\mid\bm X_t)
 \,\middle|\,
 \bm X_t^{-i}=\bm x_t^{-i}
 \right].
 \label{eq:loo-main}
\end{equation}
We therefore interpret $h_{\theta,t}^i$ as a learned approximation to this
population target, while allowing a finite checkpoint to retain some
dependence on the $i$-th coordinate of input $\bm{x}$.

To use the LOO law in the reverse update, we first convert it to the ordinary
denoising posterior. Bayes' rule gives
\begin{equation}
 q_{0\mid t}^i(x\mid\bm x_t)
 =\frac{q_{0\mid t}^i(x\mid\bm x_t^{-i})q_{t\mid0}^i(x_t^i\mid x)}
 {\sum_{\bar x}q_{0\mid t}^i(\bar x\mid\bm x_t^{-i})q_{t\mid0}^i(x_t^i\mid\bar x)}
 =:\bigl(\mathcal C_{t,x_t^i}q_{0\mid t}^i(\cdot\mid\bm x_t^{-i})\bigr)(x),
 \label{eq:conversion-main}
\end{equation}
where, for any categorical law $h$,
$(\mathcal C_{t,z}h)(x):=
h(x)q_{t\mid0}^i(z\mid x)/
\sum_{\bar x}h(\bar x)q_{t\mid0}^i(z\mid\bar x)$.
Indeed, conditional on $X_0^i=x$, corruption at coordinate $i$ is independent
of $\bm X_t^{-i}$ and has law $q_{t\mid0}^i(\cdot\mid x)$.

Finally, marginalizing the bridge under this ordinary denoising posterior
gives the exact reverse marginal
\begin{equation}
 q_{s\mid t}^i(y\mid\bm x_t)
 =\sum_x q_{s\mid0,t}^i(y\mid x,x_t^i)
 q_{0\mid t}^i(x\mid\bm x_t).
 \label{eq:exact-reverse-main}
\end{equation}
The learned reverse transition $\pi_{s|t}$ substitutes the converted learned law for the
exact posterior:
\begin{equation}
 \pi_{s\mid t}^i(y\mid\bm x_t)
 :=\sum_x q_{s\mid0,t}^i(y\mid x,x_t^i)
 \bigl(\mathcal C_{t,x_t^i}h_{\theta,t}^i
 (\cdot\mid\bm x_t)\bigr)(x),
 \label{eq:learned-reverse-main}
\end{equation}
where the dependence on $\theta$ is omitted from $\pi_{s|t}$ for simplicity.
Equations~\eqref{eq:conversion-main} and~\eqref{eq:learned-reverse-main} separate
three objects: the learned law over clean labels, its local conversion to an
ordinary denoising posterior, and the resulting learned reverse transition,
from which the next noisy token is sampled.

\subsection{\texorpdfstring{Top-$p$}{Top-p} in the LOO distribution}
\label{sec:top-p-loo}

For a categorical law $h$, write
$h(A):=\sum_{a\in A}h(a)$ and, whenever $h(A)>0$,
\begin{equation}
 (\mathsf T_Ah)(a):=\frac{h(a)\1\{a\in A\}}{h(A)}.
 \label{eq:general-restriction}
\end{equation}
Sort the vocabulary as
$h(a_1)\ge\cdots\ge h(a_V)$ with a fixed tie rule. The top-$p$ support
induced by $h$ is
\begin{equation}
  \mathsf T_ph:=\mathsf T_{A_p(h)}h, \quad
  \text{where}\ 
 A_p(h)=\{a_1,\ldots,a_{k_p(h)}\},
 \ 
 k_p(h)=\min\Bigl\{k:\sum_{j=1}^k h(a_j)\ge p\Bigr\}.
 \label{eq:top-p-support-main}
\end{equation}
Top-$p$ acts within one reverse step by substituting $\mathsf T_p h$ for $h$ in
Equation~\eqref{eq:learned-reverse-main} \citep{holtzman2020curious,hewitt2022desmoothing}; the resulting transition samples the
next noisy state (Figure~\ref{fig:concept}(a)). Even when $A_p(h)$ is a
singleton, its label is absent from the next denoiser input unless the noisy
state happens to preserve it.
Empirically, at 64 NFE on Duo-distilled with $p=0.9$, the top-$p$ support is a singleton on
$27.5\%$ of active coordinate--step pairs overall and $45.6\%$ in the final ten
reverse steps (Appendix~\ref{app:singleton-trace}).

The conversion in Equation~\eqref{eq:conversion-main} preserves hard support
restriction. For every retained set $A$ with positive normalizers,
$%
 \mathcal C_{t,x_t^i}(\mathsf T_Ah)
 =\bigl(\mathcal C_{t,x_t^i}h\bigr)(\cdot\mid A)
$ %
holds.
Indeed, writing $Q=\mathcal C_{t,x_t^i}h$, direct substitution gives
$\mathcal C_{t,x_t^i}(\mathsf T_Ah)(a)=Q(a)\1\{a\in A\}/Q(A)$, which is
$Q(\cdot\mid A)$. Thus top-$p$ is hard support restriction at the clean-label
interface used by the reverse bridge.

Can an adaptive top-$p$ rule reproduce, within one step, the update that a
visible token induces? We fix a coordinate $i$, a state $\bm x_t$ with
$z=x_t^i$, and $0\le\alpha_s<\alpha_t\le1$. For a proposal $\pi$ over clean tokens, let
$\mathsf K[\pi](y):=\sum_x q_{s\mid0,t}^i(y\mid x,z)(\mathcal C_{t,z}\pi)(x)$
denote Equation~\eqref{eq:learned-reverse-main} with $\pi$ in place of
$h:=h_{\theta,t}^i(\cdot\mid\bm x_t)$. An adaptive top-$p$ rule may choose
$p$ from the state and history and may randomize it, so its possible
transitions at this state form
$\mathcal K_h:=\operatorname{conv}\{\mathsf K[\mathsf T_ph]:0<p\le1\}$.

\begin{proposition}[Adaptive top-$p$ reproduces a proposal only under a ratio condition]
\label{prop:adaptive-reach}
For positive $h$ with $h(a_1)\ge\cdots\ge h(a_V)$ under the fixed tie rule and
any proposal $\gamma$,
\begin{equation}
 \mathsf K[\gamma]\in\mathcal K_h
 \quad\Longleftrightarrow\quad
 \frac{\gamma(a_1)}{h(a_1)}\ge\cdots\ge\frac{\gamma(a_V)}{h(a_V)}.
 \label{eq:adaptive-reach}
\end{equation}
\end{proposition}

Appendix~\ref{app:adaptive-reach-proof} gives the proof via the mixture
characterization of likelihood-ratio ordered laws \citep{jauch2021mixture}
and the invertible uniform bridge. Temperature scaling $h(a)^{1/\tau}$ with
$0<\tau\le1$ satisfies the condition, whereas the proposal \CRS{} supplies,
$h_{\theta,t}^i(\cdot\mid\widetilde{\bm x}_t)$ with stored labels visible,
need not; when the ordering fails at a visited state, no adaptive top-$p$ rule
produces the same transition. The statement concerns one update's
expressivity, not output quality.

\section{From Point Hypotheses to Persistent Context}
\label{sec:persistent-context}

Top-$p$ can collapse the clean-token support in one reverse step to a point,
but the next denoiser receives only the updated noisy state. \CRS{} instead
keeps selected clean-token values explicit in later denoiser inputs while
leaving all unselected positions revisable, paralleling the use of revealed
tokens as later context in masked diffusion
\citep{chang2022maskgit,hayakawa2026demystifying}.

\subsection{Committed reveal sampling}
\label{sec:crs-method}

Let $C\subset[L]$ be the selected coordinates and
$\bm c=(c_i)_{i\in C}$ their stored labels. We write
$\bm x[C\leftarrow\bm c]$ for replacing those coordinates in a model input.
During an initial warmup phase, \CRS{} runs the native sampler without storing
labels. It then distributes a total confidence-deficit budget $D_{\max}$ over
the remaining steps. Selecting coordinate $i$ incurs the deficit
$1-\max_a h_{\theta,t}^i(a\mid\widetilde{\bm x}_t)$.
After the warmup phase, each sampling step of \CRS{} is as follows:
\begin{enumerate}[leftmargin=*,itemsep=1pt,topsep=2pt]
  \item evaluate the denoiser on
  $\widetilde{\bm x}_t:=\bm x_t[C\leftarrow\bm c]$, so labels stored in earlier
  steps are visible;
  \item use this output for the native reverse transition, scan the currently
  unselected coordinates from highest to lowest confidence, and greedily
  select a spatially separated batch within the budget allocated to that
  step; and
  \item store the argmax label at each newly selected coordinate and update
  $(C,\bm c)$, making the selected labels visible in subsequent denoiser
  evaluations.
\end{enumerate}

We form each batch using a fixed local-exclusion rule;
Appendix~\ref{app:crs-algorithm} gives the exact constraint and connects it to
spatially dispersed scheduling
\citep{besnier2025halton,hayakawa2026demystifying}.
For each selected coordinate, \CRS{} stores the argmax of its clean-token
prediction. This value coincides with the unique retained token whenever
top-$p$ produces a singleton support. More generally, argmax is the
deterministic choice that minimizes coordinatewise Hamming risk under the
denoiser output:
\begin{equation}
 \arg\min_{c\in\mathcal S}
 \mathbb E_{X\sim h}\!\left[\1\{X\ne c\}\right]
 =\arg\max_{c\in\mathcal S}h(c).
 \label{eq:argmax-hamming-main}
\end{equation}
Sampling from the top-$p$-restricted distribution when its support is not a
singleton, and then storing that draw, would combine token-sampling randomness
with contextual feedback. \CRS{} therefore uses argmax to isolate the latter.
Indeed, at $p=0.9$ and 64 NFE it matches top-$p$ on the $45.6\%$ of active
coordinate--step pairs with singleton support in the final ten reverse steps.
The full selector and pseudocode, together with
experiments that limit how long stored labels remain visible, appear in
Appendix~\ref{app:crs-algorithm}.

\paragraph{Common final-step decoder.}
The released Duo sampler uses an ancestral endpoint
\citep{sahoo2025duo,deschenaux2026psi}. To compare support restriction and
persistent context under a shared output rule, our experiments instead count a
coordinatewise argmax readout as the final one of $N$ sampling steps. Each
step uses one denoiser evaluation, so NFE is $N$. Let
$\bm x_{\varepsilon}$ be the last noisy state on
the grid $t_0>\cdots>t_{N-1}=\varepsilon>0$. The first $N-1$ steps produce
stochastic reverse transitions, and the final step returns
$\widehat x^i=\arg\max_{a\in\mathcal S}
h_{\theta,\varepsilon}^i(a\mid\widetilde{\bm x}_{\varepsilon})$ at every
generated coordinate. This comparison protocol uses the model's clean-token
output directly rather than sampling another reverse transition.

For $A\subseteq C$, let $\bm c|_A=(c_i)_{i\in A}$ denote the corresponding
restriction of the stored labels. For CRS,
$\widetilde{\bm x}_{\varepsilon}=\bm x_{\varepsilon}
[C_{\mathrm{vis}}\leftarrow\bm c|_{C_{\mathrm{vis}}}]$, where
$C_{\mathrm{vis}}\subseteq C$ indexes coordinates
with stored labels still visible. The selected tokens appear in this final
model input but are never copied directly to the output. We restore fixed
conditioning positions separately.

\subsection{Why persistent context can coordinate parallel updates}
\label{sec:theory}

Parallel product updates can combine individually plausible tokens from
incompatible sequence-level alternatives. Correlation-aware denoisers and
joint transitions address this by modifying dependence within a step
\citep{hayakawa2025distillation,xu2024edlm,li2026codd}. \CRS{} instead retains
a product output and changes later predictions through visible context.

To isolate this effect, consider two binary coordinates with clean law
$Q_\eta=(1-\eta)\delta_{00}+\eta\delta_{11}$ for $0<\eta<1/2$.
The product of its marginals assigns total mass $2\eta(1-\eta)$ to the hybrid
strings $01$ and $10$, while $Q_\eta$ assigns them zero mass. For
$X=(X^1,X^2)\sim Q_\eta$, coordinatewise argmax selects $00$, with
sequence-level error $\mathbb P(X\ne00)=\eta$. If a selected value
$c_1\in\{0,1\}$ remains visible at the first coordinate, then
$Q_\eta(X^2=c_1\mid X^1=c_1)=1$; the other coordinate is symmetric.
Persistence, rather than argmax alone, carries this mode evidence into later
predictions.

We now generalize this mechanism. Let a latent mode
$M\in[K]$ have prior weights $w_m=\mathbb P(M=m)$. Conditional on $M=m$,
let the future coordinates $Y=(Y_1,\ldots,Y_n)$ be independent, with laws
$P_{m,i}$. Define
\begin{equation}
 Q_{\mathrm{mix}}^w(y)=\sum_{m=1}^K w_m\prod_{i=1}^nP_{m,i}(y_i),
 \qquad
 Q_{\mathrm{prod}}^w(y)=\prod_{i=1}^n\sum_{m=1}^K w_mP_{m,i}(y_i).
 \label{eq:shared-mode-laws}
\end{equation}
The mixture uses one shared mode, whereas the product law independently
remixes its coordinate marginals. Here $\KL$ denotes Kullback--Leibler (KL)
divergence, $\TC$ denotes total correlation~\citep{watanabe1960information}, a
standard measure of residual multivariate dependence whose conditional variants
have also been used in discrete diffusion~\citep{yoo2025redi}, and
$\Ent(w):=-\sum_m w_m\log w_m$. We also write
$h_2(a):=-a\log a-(1-a)\log(1-a)$.

\begin{theorem}[Shared-mode coordination]
\label{thm:shared-mode}
For the laws in Equation~\eqref{eq:shared-mode-laws},
\begin{equation}
 \KL(Q_{\mathrm{mix}}^w\kldiv Q_{\mathrm{prod}}^w)
 =\TC_{Q_{\mathrm{mix}}^w}(Y_1,\ldots,Y_n)
 \le(n-1)\Ent(w).
 \label{eq:shared-mode-gap}
\end{equation}
Let $U$ be visible context with $Y\perp\!\!\!\perp U\mid M$. For any
positive-probability value $u$, let $w_m[u]=\mathbb P(M=m\mid U=u)$ and
$\epsilon(u):=1-\max_m w_m[u]$. Then
\begin{equation}
 \KL(Q_{\mathrm{mix}}^{w[u]}\kldiv Q_{\mathrm{prod}}^{w[u]})
 \le(n-1)\{h_2(\epsilon(u))+\epsilon(u)\log(K-1)\}.
 \label{eq:context-mode-gap}
\end{equation}
\end{theorem}

Appendix~\ref{app:shared-mode-theorem-proof} gives the proof.
The theorem treats the visible context $U$ as external to $Y$, whereas \CRS{} feeds back its own argmax; the result therefore motivates persistence rather than guaranteeing the sampler's behavior.
A mode-consistent visible label concentrates $w[u]$ and drives the bound
toward zero; a poor label can instead favor the wrong mode. The appendix gives
an explicit likelihood-ratio condition and extends the calculation to combined
support restriction and persistent evidence
(Appendix~\ref{app:combined-evidence}).
Appendix~\ref{app:warmup-theory} complements this with the warmup: under the exact forward process, waiting until the state is less noisy cannot increase the oracle clean-token selection error (Proposition~\ref{prop:warmup-blackwell}).

\section{Experiments}
\label{sec:experiments}

We use the public Duo-distilled and base Duo checkpoints~\citep{sahoo2025duo}
to study two questions. First, how do top-$p$ and \CRS{} trade generative
quality for entropy across compute budgets? Second, what do paired controls
reveal about the effect of persistent context on later predictions? All
end-to-end comparisons use argmax decoding in the final step, and each policy
generates its own trajectory. Unless stated otherwise, \CRS{} uses a warmup
fraction $w=0.33$ and $D_{\max}=512$ with full clean-token support. Following prior discrete-diffusion
studies~\citep{hayakawa2025distillation,hayakawa2026demystifying}, our main
quality metric is
$\mathrm{GenPPL}=\exp(\mathrm{GenNLL})$, where GenNLL is the mean next-token
negative log-likelihood assigned to generated samples by
GPT-2-Large~\citep{radford2019language}. In the main text, we use mean
within-sequence unigram entropy as the diversity axis.
Appendix~\ref{app:experiment-details} reports complementary diversity and
repetition metrics, full configurations, experiments that limit how long
selected labels remain visible, and all uncertainties.

\subsection{How does the tradeoff change with NFE?}
\label{sec:nfe-scaling}

Panels (a)--(b) of Figure~\ref{fig:nfe-scaling} plot GenPPL against NFE while
holding four procedures fixed. On Duo-distilled, \CRS{} with top-$p$ disabled ($p=1.0$) has lower GenPPL throughout the tested NFE range than native sampling with $p\in\{1.0,0.95,0.9\}$. Base Duo shows the same ordering of point estimates in a descriptive, protocol-aligned comparison.
Appendix~\ref{app:nfe-scaling} gives estimates and uncertainties.

\begin{figure}[H]
  \centering
  \includegraphics[width=\linewidth]{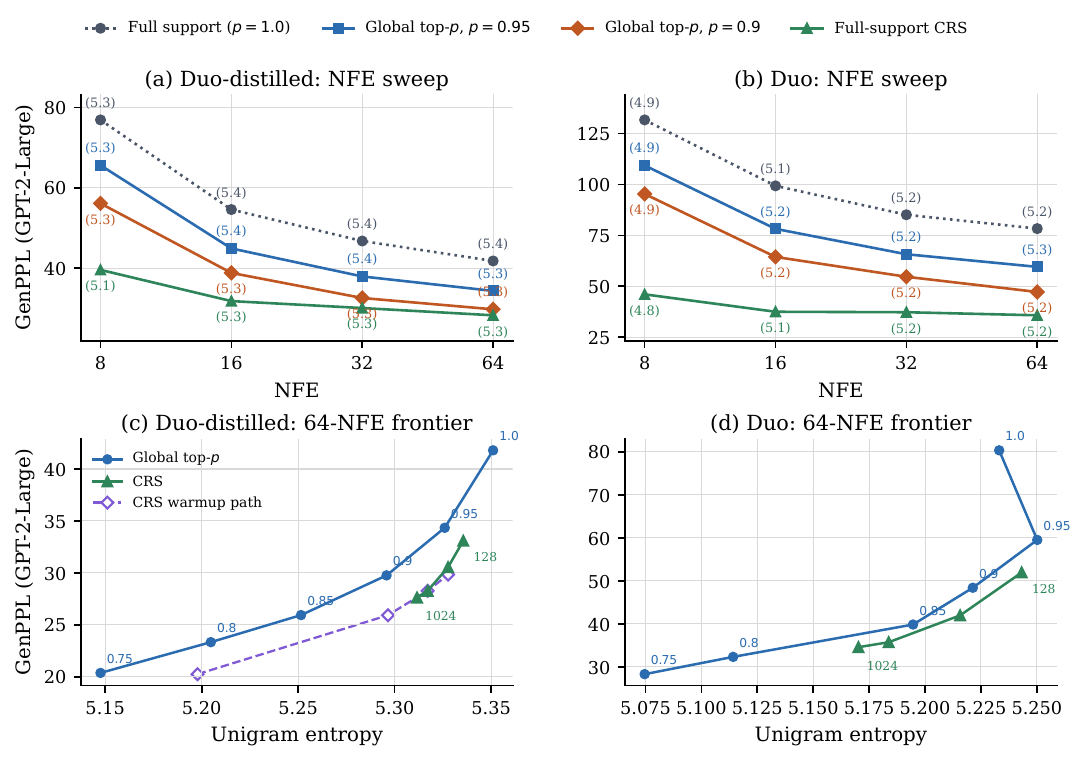}
  \caption{
    GenPPL across NFE (a--b) and 64-NFE GenPPL--entropy comparisons
    (c--d), all with argmax decoding in the final step. Parentheses in
    (a--b) give unigram entropy. Panels (c--d) sweep global top-$p$ and the
    \CRS{} selection budget; open diamonds in (c) show the
    $D_{\max}=512$ warmup path used in Figure~\ref{fig:warmup-permanent}.
    Appendix~\ref{app:nfe-scaling} gives the complete grids and uncertainties.
  }
  \label{fig:nfe-scaling}
\end{figure}

Note that our base-Duo top-$p$ values differ from those reported by
\mbox{\citet{deschenaux2026psi}} because their final step samples ancestrally,
whereas ours uses argmax decoding for every method. In a paired 32-NFE audit
that fixes all earlier reverse steps, changing only the final step back to
ancestral sampling recovers their reported range for $p\in\{1.0,0.9\}$.
Figure~\ref{fig:nfe-scaling} therefore compares the methods under our common
argmax final step. Appendix~\ref{app:nfe-scaling} gives the paired values.

\subsection{What changes when selected tokens remain visible?}
\label{sec:mechanism-evidence}

Three controls show that the path change depends on the inserted token values
and cannot be reduced to position selection or scalar temperature. Replacing
stored argmax labels by nearby alternatives reverses the likelihood gain, so
selected positions alone are insufficient. At the same mutable state, the
best scalar-temperature fit leaves
$79.4\pm2.1\%$ of the pooled KL shift unexplained, and the top-ranked token
changes on $30.2\pm0.8\%$ of audited coordinates. Restoring one
selected coordinate to its revisable input leaves the surrounding state
predictive of that label, indicating that its information has propagated to
other coordinates. Refreshing the visible labels after each step lowers raw
GenPPL further, but retaining their original values gives lower evaluator loss
at matched unigram entropy (Table~\ref{tab:refreshed-matched-app}).
Table~\ref{tab:main-controls} collects these controls;
Appendix~\ref{app:mechanism-tables} gives
the protocols. Appendices~\ref{app:storage-output-control} and
\ref{app:support-duration} test direct output copying and vary how many later
steps can see each selected label.

\begin{table}[H]
\centering\footnotesize
\setlength{\tabcolsep}{4pt}
\begin{tabular}[t]{@{}lr@{}}
\toprule
\multicolumn{2}{@{}l}{\textbf{Value intervention}, GenPPL}\\
\midrule
Stored argmax (\CRS) & $28.27\pm0.53$\\
Second-ranked token & $77.75\pm1.80$\\
Top-$p=0.9$ draw excluding argmax & $78.36\pm1.56$\\
Refreshed values & $26.42\pm0.39$\\
\bottomrule
\end{tabular}
\hfill
\begin{tabular}[t]{@{}lrr@{}}
\toprule
\multicolumn{3}{@{}l}{\textbf{Matched GenNLL residual vs.\ global top-$p$}, nats/token}\\
\midrule
Matching statistic & Immutable \CRS & Refreshed\\
\midrule
Unigram entropy & $-0.154\pm0.031$ & $-0.095\pm0.024$\\
Dist-4 & $-0.025\pm0.048$ & $-0.012\pm0.027$\\
Rep-4 & $+0.070\pm0.036$ & $+0.076\pm0.021$\\
Self-BLEU & $-0.144\pm0.019$ & $-0.163\pm0.019$\\
\bottomrule
\end{tabular}
\caption{Paired controls and matched comparisons on Duo-distilled at 64 NFE, full-support \CRS{} with $D_{\max}=512$; means $\pm$ SE over four paired groups. Left: only the stored values change, with positions, transition streams and the final step fixed. Right: \CRS{}-minus-global-top-$p$ GenNLL at the matched statistic; negative favors context.}
\label{tab:main-controls}
\end{table}

\paragraph{How do the checkpoints differ?}
At each shared uniform-forward state, we apply the same \CRS{} position-selection rule and budget separately to the two checkpoints. The resulting selected-coordinate sets overlap only moderately. Conditional on a fixed set of selected coordinates, however, the checkpoints almost always predict the same argmax labels and attain nearly identical clean-label accuracy. Duo-distilled mainly differs by selecting positions with slightly lower oracle error. This shared-state audit therefore localizes a checkpoint difference in \CRS{} position selection. Appendix~\ref{app:checkpoint-selector} gives the factorial and calibration details.

\paragraph{What changes when selection starts later?}
On exact uniform-forward states, selected-label error falls steadily as noise
decreases, consistent with Proposition~\ref{prop:warmup-blackwell}. The
sampler-level sweep varies how many steps precede selection. Later selection
recovers unigram entropy while retaining a substantial likelihood gain, and
all four warmup endpoints lie below the matched global top-$p$ curve in
Figure~\ref{fig:nfe-scaling}(c). Appendix
\ref{app:warmup-duration} gives uncertainties, controls that vary how many
later steps see each selected label, and sequence-level metrics.

\begin{figure}[H]
  \centering
  \includegraphics[width=0.68\linewidth]{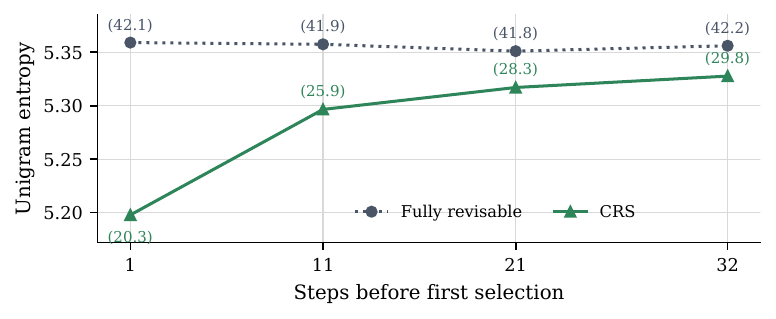}
  \caption{The \CRS{} tradeoff as selection starts later at 64 NFE.
  Ticks give steps before first selection; parentheses give GenPPL.
  Appendix~\ref{app:warmup-duration} gives uncertainties and controls that
  vary how many later steps see each selected label.}
  \label{fig:warmup-permanent}
\end{figure}

\subsection{Are support restriction and persistent context interchangeable?}
\label{sec:complete-policy-evidence}

Panels (c--d) of Figure~\ref{fig:nfe-scaling} compare CRS budget
sweeps with global top-$p$ at 64 NFE. Across the overlapping entropy region,
the CRS curve lies below the global curve on Duo-distilled; base Duo shows the
same direction in a descriptive comparison. After matching entropy within each
shard, \CRS{} at $D_{\max}=512$ lowers GenNLL by $0.154\pm0.031$ nats per
token.
Appendices~\ref{app:nfe-scaling} and~\ref{app:mechanism-tables} report the
frontier values, uncertainties, and interpolation rule. Global top-$p$ has
lower repeated 4-gram fraction (Rep-4)~\citep{welleck2020unlikelihood}; in
fixed-prefix continuation, GPT-2 favors applying $p=0.8$ at every step, whereas
Pythia~\citep{biderman2023pythia} does not distinguish the two within
uncertainty.

The curves overlap in places, but the samplers reach them differently:
top-$p$ truncates the distribution used at each current step, whereas \CRS{}
carries selected tokens into later model inputs. Delaying selection then moves
\CRS{} along its GenPPL--entropy curve. Rep-4 and fixed-prefix likelihood rank
some of these operating points differently.

\subsection{Finite retention and distributional comparison}
\label{sec:ibis-extension}
A separate set of saved outputs examines a finite-retention variant: selected
labels remain visible for the next $H=8$ model inputs, with optional release
when the current argmax disagrees with the stored label. Expiry or release
restores the sampled token to the input. This is distinct from the Refreshed control, which updates visible
values on replayed selection masks. The comparison uses 256 sequences in 16
paired shards, 63 reverse transitions and one final call, warmup $w=0.10$,
and late top-$p$ with $p=0.8$, which applies truncation only after the warmup.
These settings differ from the preceding full-support,
$w=0.33$ comparison and are reported separately.

We evaluate the final argmax output using SOL, a distance from held-out text
in a fixed encoder's feature space \citep{kornhardt2026sol}. On the global
top-$p$ sweep, stronger truncation keeps lowering GenPPL while SOL turns
upward, so lower GenPPL need not mean a closer distribution.
Figure~\ref{fig:ibis-sol-curves} shows the GPT-2-Large SOL comparison.
At the same late $p=0.8$, raw-label context, with or without release, lowers
SOL on both checkpoints under the GPT-2-Large representation;
the pointwise 95\% intervals exclude zero. No entropy matching is applied to
SOL. For the saved state without release, the intervals include zero on both
checkpoints. Appendix~\ref{app:ibis} gives the intervals and protocol. For final argmax, entropy matching favors context in point estimates,
whereas Rep-4 matching reverses the ordering.
These checks concern specified metrics and representations, not a general
ordering of text quality.

\begin{figure}[H]
  \centering
  \includegraphics[width=\linewidth]{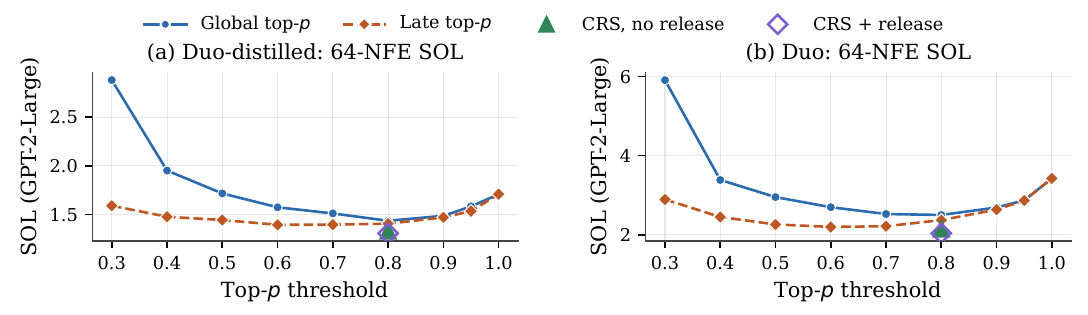}
  \caption{SOL for the final argmax output at 64 NFE. Global and late top-$p$ curves share the finite-retention cohort; CRS markers use late $p=0.8$, $H=8$, raw labels, and optional disagreement release. Lower is closer to the held-out reference in GPT-2-Large feature space. Points use 256 texts per policy. Appendix~\ref{app:ibis} reports the same-$p$ paired intervals and protocol.}
  \label{fig:ibis-sol-curves}
\end{figure}

\section{Concluding Remarks}
\label{sec:conclusion}

\CRS{} turns selected clean-token hypotheses into context for later steps,
whereas top-$p$ restricts only the distribution used at the current step. Our
analysis separates the two: adaptive truncation cannot reproduce a visible
token's update unless a ratio ordering holds, while keeping a selected token
visible can steer later product predictions toward the same mode.

Across 8--64 NFE under the same argmax final step, \CRS{} with top-$p$
disabled ($p=1.0$) achieves lower GenPPL than the fixed
$p\in\{0.95,0.9\}$ baselines on both checkpoints. At 64 NFE, varying its
selection budget gives lower GenPPL at matched entropy on Duo-distilled; a
separate base-Duo comparison shows the same trend. This advantage is specific
to the GenPPL--entropy comparison: Rep-4 and fixed-prefix likelihood can rank
the operating points differently. We have not run commitment methods for masked diffusion \citep{wu2025fastdllm,shu2026deferred} or predictor--corrector samplers \citep{deschenaux2026psi} under the shared decoder, report no human or task-based quality measure, and compare methods by NFE, not wall-clock speed. Future work should test larger checkpoints
and other uniform-state architectures.

\begin{ack}
The author thanks Chunsan Hong for helpful comments and pointers to relevant literature. The author was partially supported by JSPS KAKENHI (26K25548) and JST PRESTO (JPMJPR26K8), and received access to the ChatGPT for Academic Researchers program. 
\end{ack}

\bibliographystyle{plainnat}
\bibliography{references}

\clearpage
\appendix
\FloatBarrier

\section{Committed Reveal and Finite Persistence}
\label{app:crs-algorithm}

Let $N$ be the total NFE and let
$t_0>\cdots>t_{N-1}=\varepsilon>0$ be the noisy-state grid. The first $N-1$
sampling steps produce reverse transitions, and the final step is the common
argmax decoder from Section~\ref{sec:persistent-context}. Let $W$ be the
number of warmup steps among the first $N-1$. For a requested warmup fraction
$f$, the implementation sets
$W=\operatorname{round}\{f(N-1)\}$ and uses $R=N-1-W$ post-warmup steps.
At reverse step $r$,
$\bm x_{t_r}$ denotes the mutable noisy state and
$C_r$ the coordinates whose stored labels are currently visible. Our implementation
uses an \emph{input-overwrite-only} interface: it overwrites only the denoiser input,
\[
 \widetilde{\bm x}_r=\bm x_{t_r}[C_r\leftarrow\bm c|_{C_r}].
\]
It does not pin the noisy-state token used by the reverse transition. Thus, after any scheduled
top-$p_r$ restriction, the actual one-coordinate transition is
\begin{equation}
 \pi_{t_{r+1}\mid t_r}^{i,\mathrm{ow}}
 (y\mid\widetilde{\bm x}_r;x_{t_r}^i)
 :=
 \sum_x q_{t_{r+1}\mid0,t_r}^i(y\mid x,x_{t_r}^i)
 \Bigl(
 \mathcal C_{t_r,x_{t_r}^i}
 \mathsf T_{p_r}h_{\theta,t_r}^i
 (\cdot\mid\widetilde{\bm x}_r)
 \Bigr)(x).
 \label{eq:overwrite-transition-app}
\end{equation}
The superscript ``ow'' denotes this input-overwrite interface. The semicolon
separates the overwritten denoiser input from the mutable current-token
argument: $h_{\theta,t_r}^i$ sees $\widetilde{\bm x}_r$, whereas the conversion
and bridge in Equations~\eqref{eq:conversion-main}
and~\eqref{eq:learned-reverse-main} continue to use $x_{t_r}^i$. Here
$\mathsf T_{p_r}h:=\mathsf T_{A_{p_r}(h)}h$ and $\mathsf T_1$ is the
identity. For each eligible coordinate,
\begin{equation}
 a_r^i\in\arg\max_{a\in\mathcal S}h_{\theta,t_r}^i(a\mid\widetilde{\bm x}_r),
 \qquad
 d_r^i=1-h_{\theta,t_r}^i(a_r^i\mid\widetilde{\bm x}_r).
 \label{eq:selector-deficit-app}
\end{equation}
The canonical selector scans currently unselected coordinates in increasing
$d_r^i$ and greedily forms a within-step batch $S_r$. With fixed spacing
radius $R_{\mathrm{sp}}$, every accepted batch satisfies
\begin{equation}
 |S_r|\le K_r,
 \qquad
 \sum_{i\in S_r}d_r^i\le E_r,
 \qquad
 |i-j|>R_{\mathrm{sp}}
 \quad\text{for all distinct }i,j\in S_r.
 \label{eq:selector-constraints-app}
\end{equation}
The spacing test is applied among coordinates newly selected at step $r$;
coordinates stored at earlier steps are excluded from reselection but do not
exclude their neighboring positions at later steps. This local-exclusion rule
is a simple one-dimensional analogue of spatially dispersed unmasking in the
Halton scheduler~\citep{besnier2025halton}. It also acts as a hard proxy for
the spatial-dispersion term in the choose-then-sample decomposition of
Proposition~4 in \citet{hayakawa2026demystifying}: confidence ordering favors
low-uncertainty coordinates, while spacing prevents one batch from being spent
entirely on nearby, potentially redundant positions. This connection motivates
the fixed rule rather than identifying an optimal radius for language.
The primary implementation distributes the path-total deficit budget and the
sequence-length count budget uniformly over the $R$ post-warmup steps:
\[
 E_r=D_{\max}/R,
 \qquad
 K_r=\lceil L/R\rceil,
 \qquad r=W,\ldots,N-2.
\]
The deficit budgets are upper bounds rather than quotas. In the primary
64-NFE duration factorial, $K_r=25$ after warmup and
$R_{\mathrm{sp}}=4$.
Smaller deficit means that the denoiser output places more mass on its preferred
label. The selector is a concentration controller rather than a calibrated
estimate of target-token error.

\begin{algorithm}[H]
\caption{Committed reveal with horizon $H$}
\label{alg:crs-horizon}
\small
\begin{algorithmic}[1]
\Require checkpoint $\mathsf M_\theta$; total NFE $N$ and grid
$t_0>\cdots>t_{N-1}=\varepsilon>0$;
warmup $W$; horizon $H\in\{0,1,\ldots,\infty\}$; selector budgets $E_r$ and
count caps $K_r$;
top-$p$ schedule $(p_r)_r$
\State sample $\bm x_{t_0}$ from the terminal uniform law
\State $C\gets\varnothing$; stored labels $\bm c\gets\varnothing$;
expiry map $e\gets\varnothing$; previously selected set $B\gets\varnothing$
\For{$r=0,\ldots,N-2$}
  \If{$0<H<\infty$}
    \State remove $i$ from $C$ whenever $r=e_i$
  \EndIf
  \State $\widetilde{\bm x}_r\gets\bm x_{t_r}[C\leftarrow\bm c|_C]$
  \State $h_r\gets\mathsf M_\theta(\widetilde{\bm x}_r,t_r)$
  \If{$r\ge W$ and $H>0$}
    \State choose $S_r$ from $[L]\setminus B$ using
    Equations~\eqref{eq:selector-deficit-app}--\eqref{eq:selector-constraints-app}
    \State $c_i\gets a_r^i$ for $i\in S_r$
    \If{$0<H<\infty$}
      \State $e_i\gets r+H+1$ for $i\in S_r$
    \EndIf
    \State $C\gets C\cup S_r$; $B\gets B\cup S_r$
  \EndIf
  \State independently sample each updated $x_{t_{r+1}}^i$ from
  Equation~\eqref{eq:overwrite-transition-app}
  \Comment{the transition uses $x_{t_r}^i$, not $\widetilde x_r^i$}
\EndFor
\If{$0<H<\infty$}
  \State remove $i$ from $C$ whenever $N-1=e_i$
\EndIf
\State use the final step at $t_{N-1}$ for the common argmax readout,
with input $\bm x_{t_{N-1}}[C\leftarrow\bm c|_C]$
\end{algorithmic}
\end{algorithm}

Newly stored labels do not alter the already computed transition at step $r$;
they first enter the denoiser at step $r+1$. With $e_i=r+H+1$, a label selected
at $r$ is visible for exactly the next $H$ steps (or all remaining steps if
fewer remain). For $H=0$, selection and storage are disabled and the algorithm
is the no-storage revisable path under the supplied top-$p$ schedule. For
$H=\infty$, no expiry occurs and the
algorithm is \CRS. In the finite-horizon factorial, an expired
coordinate becomes revisable but remains in $B$, so it cannot be selected
again. Thus $H=8$ is a finite-persistence control lasting eight later steps, not a
renewable policy.

\FloatBarrier

\section{Adaptive Top-\texorpdfstring{$p$}{p} and the Uniform Bridge}
\label{app:adaptive-reach}

\subsection{Proof of Proposition~\ref{prop:adaptive-reach}}
\label{app:adaptive-reach-proof}

We fix the coordinate, the mutable token $z=x_t^i$ and
$0\le\alpha_s<\alpha_t\le1$ as in Section~\ref{sec:top-p-loo}, and write
$\ell_z(a):=q_{t\mid0}^i(z\mid a)>0$. By Equation~\eqref{eq:bridge-main},
\begin{equation}
 \mathsf K[\pi](y)=\sum_a
 (\mathcal C_{t,z}\pi)(a)\,B(y,a),
 \qquad
 B(y,a):=q_{s\mid0,t}^i(y\mid a,z)
 =\frac{q_{s\mid0}^i(y\mid a)\,q_{t\mid s}^i(z\mid y)}{\ell_z(a)}.
 \label{eq:reach-factorization}
\end{equation}
As a matrix,
$B=\operatorname{diag}\bigl(q_{t\mid s}^i(z\mid\cdot)\bigr)\,Q_s\,
\operatorname{diag}(\ell_z^{-1})$ with
$Q_s=(1-\alpha_s)I+(\alpha_s/V)\bm 1\bm 1^{\top}$. Both diagonal factors are
strictly positive and $Q_s$ has eigenvalues $1$ and $1-\alpha_s>0$, so $B$ is
invertible. Hence $\mathsf K[\gamma]$ is a convex combination of the
transitions $\mathsf K[\mathsf T_ph]$ if and only if
$\sigma:=\mathcal C_{t,z}\gamma$ is the same convex combination of the
converted restrictions $\mathcal C_{t,z}(\mathsf T_ph)$.

Let $A_k=\{a_1,\ldots,a_k\}$ under the fixed ordering of $h$, with the same
tie order throughout; every $A_k$ is realized by some threshold $p$, and these
are the only retained sets. Bayes
reweighting commutes with restriction
(Section~\ref{sec:top-p-loo}), so with $r:=\mathcal C_{t,z}h$ and
$R_k:=\sum_{j\le k}r(a_j)$,
$\mathcal C_{t,z}(\mathsf T_{A_k}h)(a_j)=r(a_j)\1\{j\le k\}/R_k$.
If $\sigma=\sum_kw_k\,\mathcal C_{t,z}(\mathsf T_{A_k}h)$ with weights
$w_k\ge0$ summing to one, then
$\sigma(a_j)/r(a_j)=\sum_{k\ge j}w_k/R_k$, and differencing gives the unique
weights
\begin{equation}
 w_k=R_k\left\{\frac{\sigma(a_k)}{r(a_k)}-
 \frac{\sigma(a_{k+1})}{r(a_{k+1})}\right\},
 \qquad
 \frac{\sigma(a_{V+1})}{r(a_{V+1})}:=0.
 \label{eq:reach-prefix-weights}
\end{equation}
Conversely, these weights reconstruct $\sigma$ and sum to
$\sum_j\sigma(a_j)=1$, so $\sigma$ is such a mixture exactly when all $w_k$
are nonnegative, that is, when $\sigma/r$ is non-increasing along
$(a_1,\ldots,a_V)$. Since $\sigma(a)/r(a)$ equals $\gamma(a)/h(a)$ times a
positive constant, this is Equation~\eqref{eq:adaptive-reach}. The mixture
characterization of likelihood-ratio ordered laws is classical
\citep{jauch2021mixture}; invertibility of the uniform bridge transfers it to
reverse transitions.

For temperature scaling, $h_\tau(a)\propto h(a)^{1/\tau}$ with
$0<\tau\le1$ gives $h_\tau(a)/h(a)\propto h(a)^{1/\tau-1}$, which is
non-increasing along $(a_1,\ldots,a_V)$, so
$\mathsf K[h_\tau]\in\mathcal K_h$; matching such an update may require
randomizing $p$. When the stored labels are a function of the sampled-state
history, a violation of Equation~\eqref{eq:adaptive-reach} at a state visited
with positive probability prevents adaptive top-$p$ from reproducing the
complete state-trajectory law of \CRS{}, although terminal marginals can still
coincide.

\FloatBarrier

\section{Warmup as Lower-Risk Value Selection}
\label{app:warmup-theory}

\subsection{Why wait before selecting tokens?}

The CRS warmup introduced in Section~\ref{sec:crs-method} changes two
quantities: it can make the selected value more reliable,
and it leaves fewer later steps in which that value can affect the path, a
tension also reflected in recent confidence- and trajectory-based commitment
policies~\citep{shu2026deferred,sun2026pathmatters}. The
following proposition isolates the first effect at an exact forward state. For
coordinate $i$, define the oracle LOO Hamming Bayes risk at noise level $t$ by
\begin{equation}
 \mathcal R_t^i
 :=\mathbb E\!\left[1-\max_{a\in\mathcal S}
 q_{0\mid t}^i(a\mid\bm X_t^{-i})\right].
 \label{eq:loo-warmup-risk}
\end{equation}

\begin{proposition}[Lower noise weakly reduces oracle selection error]
\label{prop:warmup-blackwell}
For the coordinate-factorized uniform forward process and noise levels
$0\le\alpha_s<\alpha_t\le1$,
$\mathcal R_s^i\le\mathcal R_t^i$ for every coordinate $i$.
\end{proposition}

Appendix~\ref{app:warmup-proposition-proof} gives the proof.
Additional corruption maps $\bm X_s^{-i}$ to $\bm X_t^{-i}$, so the cleaner
context can simulate any predictor based on the noisier one. The proposition
concerns the exact posterior; in practice, we rank positions using the learned
model's confidence deficit. Changing the warmup also changes how many later
steps can use a selected token, so the comparison in Figure~\ref{fig:warmup-permanent} measures both effects together.

\subsection{Proof of Proposition~\ref{prop:warmup-blackwell}}
\label{app:warmup-proposition-proof}

For $s<t$, coordinatewise factorization of the forward kernel gives the
Markov chain
\[
 X_0^i\longrightarrow\bm X_s^{-i}\longrightarrow\bm X_t^{-i}.
\]
Any decision rule based on $\bm X_t^{-i}$ can therefore be simulated from
$\bm X_s^{-i}$ by first drawing a synthetic later state from
$q_{t\mid s}^{-i}:=\prod_{j\ne i}q_{t\mid s}^j$, the product transition on
the non-$i$ coordinates. The joint law of that synthetic state and $X_0^i$ is the
same as the joint law of $\bm X_t^{-i}$ and $X_0^i$. Hence the best rule based
on $\bm X_s^{-i}$ cannot have larger zero--one risk. The Bayes rule is the
largest-probability LOO label, so
$\mathcal R_s^i\le\mathcal R_t^i$.

\subsection{Illustration: warmup separates sequence modes}
\label{app:warmup-codebook}

The proposition orders exact-law selection risks without assuming a particular
data distribution. To make its sequence-level consequence concrete, consider
$K$ candidate clean sequences $c_1,\ldots,c_K\in\mathcal S^L$ with minimum
Hamming separation
\[
 \Delta:=\min_{z\ne z'}d_H(c_z,c_{z'})>0.
\]
Let the mode $M$ be uniform on $[K]$, set $\bm X_0=c_M$, and condition on
$M=m_\star$, writing $c_\star:=c_{m_\star}$. For a forward-corrupted state
$\bm x_s$, put $d=d_H(\bm x_s,c_\star)$. Under the uniform channel,
\[
 q_{s\mid0}(\bm x_s\mid c_z)
 =v_s^{L-d_H(\bm x_s,c_z)}u_s^{d_H(\bm x_s,c_z)},
 \qquad
 v_s=1-\alpha_s+\frac{\alpha_s}{V},
 \quad u_s=\frac{\alpha_s}{V}.
\]
Because $v_s>u_s$ away from complete noise, equal-prior
maximum-a-posteriori (MAP) decoding is nearest-Hamming decoding. For every
competing $c_z$, the triangle inequality gives
\[
 d_H(\bm x_s,c_z)
 \ge d_H(c_z,c_\star)-d_H(\bm x_s,c_\star)
 \ge\Delta-d.
\]
Thus $d<\Delta/2$ makes $c_\star$ the unique nearest codeword and hence the
unique MAP mode.

This mode decision also connects to the coordinatewise argmax used by
\CRS{}. Let $w_z=\mathbb P(M=z\mid\bm X_s=\bm x_s)$, suppose
$w_\star:=w_{m_\star}>1/2$, and define
$R_w(a)=\sum_z w_zd_H(a,c_z)$. The reverse triangle inequality yields
\begin{equation}
 R_w(a)-R_w(c_\star)
 \ge(2w_\star-1)d_H(a,c_\star).
 \label{eq:hamming-bayes-gap-app}
\end{equation}
The right-hand side is positive for every $a\ne c_\star$, proving uniqueness
of the posterior-Hamming Bayes action. Because Hamming risk separates by
coordinate, this action is also obtained by taking the clean-token posterior
argmax at each coordinate.

The same model gives an explicit probability of entering this uniquely
decodable region. Let
\[
 r_s=\frac{(V-1)\alpha_s}{V}
\]
be the probability that one observed symbol differs from its clean symbol.
Then $d_H(\bm X_s,c_\star)\sim\operatorname{Binomial}(L,r_s)$. If
$r_s<\Delta/(2L)$, the standard binomial Chernoff bound gives
\begin{equation}
 \mathbb P\!\left(
 c_\star\text{ is not the unique nearest-Hamming mode}
 \mid c_\star
 \right)
 \le
 \exp\!\left\{
 -L\,d_{\mathrm{bin}}\!\left(
 \frac{\Delta}{2L}\,\middle\Vert\,r_s
 \right)
 \right\},
 \label{eq:warmup-codebook-bound-app}
\end{equation}
where
$d_{\mathrm{bin}}(a\,\Vert\,b)
=a\log(a/b)+(1-a)\log((1-a)/(1-b))$.
Thus lower noise both improves the general oracle-risk ordering in
Proposition~\ref{prop:warmup-blackwell} and, in this separated-mode example,
makes one coherent sequence mode more likely to be uniquely identifiable.
The codebook calculation supplies a concrete sufficient condition, while the
proposition itself provides the distribution-free exact-law comparison.

\FloatBarrier

\section{Shared-Mode Proof}
\label{app:shared-mode-proof}

\subsection{Proof of Theorem~\ref{thm:shared-mode}}
\label{app:shared-mode-theorem-proof}

We first make the conditional extension precise. Let $M\in[K]$ have
$\mathbb P(M=m)=w_m$ and suppose
\[
 \mathbb P(Y=y\mid M=m)
 =\prod_{i=1}^nP_{m,i}(y_i).
\]
Let $U$ be a potential visible-context observation and assume
$Y\perp\!\!\!\perp U\mid M$. For a realization $u$, define the context
likelihood $L_m(u):=\mathbb P(U=u\mid M=m)$. Whenever the denominator is
positive, Bayes' rule gives
\begin{equation}
 w_m[u]:=\mathbb P(M=m\mid U=u)
 =\frac{w_mL_m(u)}{\sum_{\ell=1}^K w_\ell L_\ell(u)}.
 \label{eq:context-posterior-weights-app}
\end{equation}
The conditional-independence assumption is essential: without it,
conditioning on arbitrary context can create dependence among the $Y_i$ even
when the mode is fixed.

The one-coordinate marginals of $Q_{\mathrm{mix}}^w$ are
$Q_{\mathrm{mix},i}^w(y_i)=\sum_m w_mP_{m,i}(y_i)$, hence
$Q_{\mathrm{prod}}^w=\prod_iQ_{\mathrm{mix},i}^w$. Therefore
\begin{equation}
 \KL(Q_{\mathrm{mix}}^w\kldiv Q_{\mathrm{prod}}^w)
 =\sum_i\Ent_{Q_{\mathrm{mix}}^w}(Y_i)
  -\Ent_{Q_{\mathrm{mix}}^w}(Y)
 =:\TC_{Q_{\mathrm{mix}}^w}(Y_1,\ldots,Y_n).
\end{equation}
Using the chain rule and conditional independence given $M$,
\begin{align}
 \TC_{Q_{\mathrm{mix}}^w}(Y_1,\ldots,Y_n)
 &=\sum_{i=2}^n I(Y_i;Y_{<i})
 \le\sum_{i=2}^n I(Y_i;M)
 \le(n-1)\Ent(M)=(n-1)\Ent(w).
\end{align}
Equation~\eqref{eq:context-posterior-weights-app} and
$Y\perp\!\!\!\perp U\mid M$ show that, after observing $U=u$, the law
of $Y$ has the same mixture-of-products form with $w$ replaced by
$w[u]$. Thus $Q_{\mathrm{mix}}^{w[u]}$ and
$Q_{\mathrm{prod}}^{w[u]}$ are precisely
the two laws in Equation~\eqref{eq:shared-mode-laws} evaluated at these
posterior weights. Repeating the argument gives the conditional total-correlation
bound. If $\epsilon(u)=1-\max_m w_m[u]$, separating the largest mode from
the remaining $K-1$ modes yields
\[
 \Ent(w[u])
 \le h_2(\epsilon(u))+\epsilon(u)\log(K-1),
\]
which proves Equation~\eqref{eq:context-mode-gap}.

For completeness, a simple likelihood-ratio condition makes the concentration
explicit. Fix a reference mode $m^\star$ with $w_{m^\star}>0$ and
$L_{m^\star}(u)>0$, and put
$\rho(u):=\max_{m\ne m^\star}L_m(u)/L_{m^\star}(u)$. Summing posterior odds in
Equation~\eqref{eq:context-posterior-weights-app} gives
\begin{equation}
 \frac{1-w_{m^\star}[u]}{w_{m^\star}[u]}
 \le\frac{1-w_{m^\star}}{w_{m^\star}}\rho(u).
 \label{eq:context-posterior-odds-app}
\end{equation}
Independent mode-consistent context observations multiply their likelihood
ratios, so a uniform ratio below one makes this bound decay exponentially in
the number of observations.
This result quantifies coordination around the mode favored by the observed
context, not whether that mode is the correct one. It also treats $U$ as an
exogenous observation satisfying the stated conditional-independence model;
the closed loop created when a sampler inserts its own argmax is assessed by
the empirical controls in Appendix~\ref{app:mechanism-tables}.

\subsection{Support evidence and persistent evidence}
\label{app:combined-evidence}

The theorem treats visible labels as evidence about a shared mode. We now place
one-step support restriction and later persistent context in the same posterior
calculation. Condition throughout on the current reverse history, so the
retained sets below are fixed. We use $U_{1:H}$ as an additive-evidence
abstraction for the step-specific states encountered while selected context
remains visible; it is not a model of independently observing the same stored
token $H$ times. Let $M\in[K]$ have prior weights $w_m$. Conditional on $M=m$,
suppose a support-bearing vector $X=(X_1,\ldots,X_d)$ has product law
$\prod_iR_{m,i}$, a future vector $Y=(Y_1,\ldots,Y_n)$ has product law
$\prod_jP_{m,j}$, and visible-context observations $U_1,\ldots,U_H$ are
mutually independent. Assume that $X$, $Y$, and $(U_e)_{e=1}^H$ are mutually
independent given $M$.

Fix a reference mode $m^\star$ and a retained rectangle $A=\prod_iA_i$. For
$m\ne m^\star$ and a realization $u_{1:H}$, define
\begin{align}
 S_m(A)
 &:=\sum_i\log\frac{R_{m^\star,i}(A_i)}{R_{m,i}(A_i)},
 &
 T_m(u_{1:H})
 &:=\sum_{e=1}^H
 \log\frac{\mathbb P(U_e=u_e\mid M=m^\star)}
 {\mathbb P(U_e=u_e\mid M=m)}.
 \label{eq:combined-evidence-parts-app}
\end{align}
Assume that these likelihood ratios are finite and positive, and put
\begin{equation}
 S(A):=\min_{m\ne m^\star}S_m(A),
 \qquad
 T(u_{1:H}):=\min_{m\ne m^\star}T_m(u_{1:H}).
 \label{eq:combined-evidence-minima-app}
\end{equation}

\begin{theorem}[Combined mode evidence]
\label{thm:combined-evidence}
Let $w[A,u]$ be the posterior mode law after observing $X\in A$ and
$U_{1:H}=u_{1:H}$. Then
\begin{equation}
 \frac{1-w_{m^\star}[A,u]}{w_{m^\star}[A,u]}
 \le
 \frac{1-w_{m^\star}}{w_{m^\star}}
 e^{-(S(A)+T(u_{1:H}))}.
 \label{eq:combined-odds-app}
\end{equation}
With
\begin{equation}
 a:=\frac{1-w_{m^\star}}{w_{m^\star}},
 \qquad
 \varepsilon(S,T):=\frac{ae^{-(S+T)}}{1+ae^{-(S+T)}},
 \qquad
 \bar\varepsilon:=\min\{\varepsilon(S(A),T(u_{1:H})),1-1/K\},
 \label{eq:combined-error-app}
\end{equation}
we have
$1-w_{m^\star}[A,u]\le\varepsilon(S(A),T(u_{1:H}))$ and
\begin{equation}
 \Ent(w[A,u])
 \le h_2(\bar\varepsilon)+\bar\varepsilon\log(K-1).
 \label{eq:combined-entropy-app}
\end{equation}
If $Q_{\mathrm{mix}}^{w[A,u]}$ is the mixture of the product laws of $Y$ under
$w[A,u]$ and $Q_{\mathrm{prod}}^{w[A,u]}$ is the product of its coordinate
marginals, then
\begin{equation}
 \KL(Q_{\mathrm{mix}}^{w[A,u]}\kldiv Q_{\mathrm{prod}}^{w[A,u]})
 \le(n-1)
 \left[h_2(\bar\varepsilon)+\bar\varepsilon\log(K-1)\right].
 \label{eq:combined-gap-app}
\end{equation}
\end{theorem}

\begin{proof}
For every $m\ne m^\star$, Bayes' rule and conditional independence give
\[
 \frac{w_m[A,u]}{w_{m^\star}[A,u]}
 =\frac{w_m}{w_{m^\star}}
 \exp\{-S_m(A)-T_m(u_{1:H})\}
 \le\frac{w_m}{w_{m^\star}}
 e^{-(S(A)+T(u_{1:H}))}.
\]
Summing proves Equation~\eqref{eq:combined-odds-app}, and solving the odds
inequality gives the posterior-error bound. When that upper bound is below
$1-1/K$, entropy is maximized by placing the remaining mass uniformly over
the other $K-1$ modes. For a larger upper bound, the uniform mode law is
feasible and the entropy certificate saturates at $\log K$. This proves
Equation~\eqref{eq:combined-entropy-app}. Applying
Theorem~\ref{thm:shared-mode} to the future product components gives
Equation~\eqref{eq:combined-gap-app}.
\end{proof}

Equation~\eqref{eq:combined-odds-app} also gives the mode-log-loss
certificate
\begin{equation}
 -\log w_{m^\star}[A,u]
 \le
 \log\!\left(1+a e^{-\{S(A)+T(u_{1:H})\}}\right)
 =:\mathcal L_{\mathrm{mode}}\!\left(S(A),T(u_{1:H})\right).
 \label{eq:mode-logloss-certificate-app}
\end{equation}

\begin{corollary}[Diminishing returns under additive mode evidence]
\label{cor:mode-diminishing-returns}
Assume $0<w_{m^\star}<1$. Write
$a=(1-w_{m^\star})/w_{m^\star}$ and use
$\mathcal L_{\mathrm{mode}}$ from
Equation~\eqref{eq:mode-logloss-certificate-app}.
For $\delta>0$, let the gain in this mode-log-loss certificate be
$\mathcal G_\delta(S,T):=\mathcal L_{\mathrm{mode}}(S,T)
-\mathcal L_{\mathrm{mode}}(S,T+\delta)$. Then
\begin{equation}
 \mathcal G_\delta(S,T)>0,
 \qquad
 \partial_S\mathcal G_\delta(S,T)
 =\partial_T\mathcal G_\delta(S,T)<0,
 \label{eq:mode-diminishing-returns}
\end{equation}
and $\mathcal G_\delta(S,T)\to0$ as $S+T\to\infty$.
\end{corollary}

\begin{proof}
Put $u=S+T$. Direct substitution gives
\[
 \mathcal G_\delta(S,T)
 =\log\frac{1+ae^{-u}}{1+ae^{-(u+\delta)}}>0.
\]
Because the gain depends on $S$ and $T$ only through their sum, its two partial
derivatives are equal. They are
\begin{equation}
 \partial_S\mathcal G_\delta(S,T)
 =\partial_T\mathcal G_\delta(S,T)
 =-\frac{ae^{-u}}{1+ae^{-u}}
 +\frac{ae^{-(u+\delta)}}{1+ae^{-(u+\delta)}}<0.
\end{equation}
The inequality follows because $x/(1+x)$ is strictly increasing. The displayed
ratio defining $\mathcal G_\delta$ tends to one as $u\to\infty$, so the gain
tends to zero.
\end{proof}

Under uniform corruption, if $g_{m,j}$ is the clean-token law at visible
coordinate $j$, token $c_j$ supplies Bayes factor
\begin{equation}
 B_{t,j}^{m^\star,m}(c_j)
 =\frac{(1-\alpha_t)g_{m^\star,j}(c_j)+\alpha_t/V}
 {(1-\alpha_t)g_{m,j}(c_j)+\alpha_t/V}.
 \label{eq:uniform-anchor-bf-app}
\end{equation}
When $g_{m^\star,j}(c_j)>g_{m,j}(c_j)$, this factor increases as
$\alpha_t$ decreases. A mode-consistent visible token therefore becomes more
informative toward the clean end of the reverse path.

\FloatBarrier

\section{Experimental Details}
\label{app:experiment-details}

All model runs were inference-only and used pretrained checkpoints without
additional training. Jobs ran on individual internal NVIDIA A100 GPUs with 40 or 80 GB
of memory, with independent jobs parallelized across available GPUs.
We estimate approximately 45 A100 GPU-hours for the main generation (38.5 hours) and
evaluations (6.5 hours). Because the runtime of each experiment is roughly proportional
to its total NFE count, and some data points appear in more than one figure, we report
only the total; we did not measure per-run or per-sampler wall-clock time.

\subsection{Pathwise singleton trace}
\label{app:singleton-trace}

We trace the native Duo-distilled sampler at 64 NFE using four seed shards and
four paths per shard. To make the trace quantities explicit, at reverse step
$r$ define
\begin{equation}
 h_r^i(a):=h_{\theta,t_r}^i(a\mid\bm x_{t_r}),
 \qquad
 A_{r,i}^{(p)}:=A_p(h_r^i),
 \qquad
 a_r^i:=\arg\max_a h_r^i(a).
 \label{eq:singleton-trace-def-app}
\end{equation}
A singleton event is
$E_{r,i}^{(p)}:=\{A_{r,i}^{(p)}=\{a_r^i\}\}$. On this event,
$\mathsf T_ph_r^i=\delta_{a_r^i}$ and Bayes conversion preserves that point
mass. Hence the transition sampled by the current stochastic bridge reduces to
\begin{equation}
 X_{t_{r+1}}^i\sim
 q_{t_{r+1}\mid0,t_r}^i
 \bigl(\,\cdot\mid a_r^i,x_{t_r}^i\bigr),
 \label{eq:singleton-bridge-app}
\end{equation}
which need not be a point mass even though the clean-token proposal is one.

The table reports averages of the following event indicators. ``Next step''
is
\begin{equation}
 \1\!\left\{
 \arg\max_a h_{r+1}^i(a)=a_r^i
 \right\}
 \quad\text{conditioned on }E_{r,i}^{(p)}
 \text{ and on step }r+1\text{ existing},
 \label{eq:singleton-next-step-app}
\end{equation}
so it asks whether the point hypothesis remains the denoiser's preferred clean
label after one sampled bridge transition; it does not require the next top-$p$
support to remain a singleton. ``Final'' is
\begin{equation}
 \1\{\widehat x^i=a_r^i\},
 \qquad
 \widehat x^i=\arg\max_a
 h_{\theta,\varepsilon}^i(a\mid\bm x_\varepsilon),
 \label{eq:singleton-final-app}
\end{equation}
conditioned on $E_{r,i}^{(p)}$. ``Stochastic bridge'' is the conditional rate
of
\begin{equation}
 \max_y q_{t_{r+1}\mid0,t_r}^i
 (y\mid a_r^i,x_{t_r}^i)<1-10^{-8},
 \label{eq:singleton-stochastic-bridge-app}
\end{equation}
which distinguishes a point clean hypothesis from a deterministic update of
the mutable noisy state. ``Singleton'' averages $\1\{E_{r,i}^{(p)}\}$ over all
active coordinate--step pairs; ``Last 10'' restricts that average to the final
ten reverse steps.

\begin{table}[H]
\centering
\small
\setlength{\tabcolsep}{2.8pt}
\begin{tabular}{@{}lccccc@{}}
\toprule
Native $p$ & Singleton & Last 10 & Next step & Final & Stochastic bridge \\
\midrule
0.8 & $35.44\!\pm\!0.32$ & $58.46\!\pm\!0.90$ & $99.65\!\pm\!0.03$ &
$95.93\!\pm\!0.20$ & $94.80\!\pm\!0.06$ \\
0.9 & $27.54\!\pm\!0.43$ & $45.56\!\pm\!0.24$ & $99.78\!\pm\!0.02$ &
$97.16\!\pm\!0.12$ & $94.76\!\pm\!0.09$ \\
\bottomrule
\end{tabular}
\caption{Pathwise point hypotheses, in percent (mean$\pm$standard error (SE)
over four seed
shards). `Last 10' is the singleton rate over the final ten reverse
transitions.}
\label{tab:singleton-trace-app}
\end{table}

\subsection{NFE scaling on the two Duo checkpoints}
\label{app:nfe-scaling}

The common-decoder NFE study uses four aligned seed shards and 64 samples per
cell per shard. Both public Duo checkpoints generate 1024 GPT-2 tokens with
path-total $D_{\max}=512$ for CRS. In both cases the warmup fraction is $0.33$,
CRS uses full clean-token support and persistent input visibility ($H=\infty$), and every
procedure uses argmax decoding in the final step. Initial states and
random-number-generator (RNG) streams
are reset within each NFE. The NFE sweep sets
$K_r=\lceil1024/(N-1-W)\rceil$, giving per-step caps $205,103,49,25$ at
$N\in\{8,16,32,64\}$ and approximately 1024 positions of path capacity after
warmup. The subtraction by one reserves the final step for the
common argmax readout rather than a reverse transition.

At 64 NFE, the first selection follows 21 warmup steps. With
$D_{\max}=512$, CRS cumulatively selects $991.23\pm1.02$ coordinates per sample
on Duo-distilled and $963.73\pm0.73$ on Duo. Across the 42 post-warmup steps,
the mean number of simultaneously visible stored labels is
$508.23\pm0.19$ and $502.68\pm0.16$, respectively. Values are mean$\pm$SE over
the four seed shards. These path statistics show that the intervention is
gradually accumulated rather than a one-shot overwrite of the sequence.

\begin{table}[H]
\centering
\small
\setlength{\tabcolsep}{4.2pt}
\begin{tabular}{@{}lrrrr@{}}
\toprule
Checkpoint / NFE & Full support & top-$p$ 0.95 & top-$p$ 0.9 & CRS \\
\midrule
Duo-distilled / 8  & 76.89 / 5.269 & 65.60 / 5.282 & 56.11 / 5.255 & 39.57 / 5.136 \\
Duo-distilled / 16 & 54.58 / 5.376 & 44.91 / 5.362 & 38.82 / 5.324 & 31.80 / 5.318 \\
Duo-distilled / 32 & 46.76 / 5.380 & 37.94 / 5.360 & 32.57 / 5.320 & 30.08 / 5.338 \\
Duo-distilled / 64 & 41.79 / 5.351 & 34.34 / 5.326 & 29.76 / 5.296 & 28.27 / 5.317 \\
\midrule
Duo / 8  & 131.66 / 4.897 & 109.31 / 4.938 & 95.32 / 4.943 & 46.04 / 4.755 \\
Duo / 16 &  99.31 / 5.138 &  78.18 / 5.189 & 64.41 / 5.176 & 37.49 / 5.105 \\
Duo / 32 &  85.11 / 5.216 &  65.72 / 5.242 & 54.66 / 5.229 & 37.26 / 5.185 \\
Duo / 64 &  78.31 / 5.231 &  59.53 / 5.250 & 47.20 / 5.220 & 35.75 / 5.184 \\
\bottomrule
\end{tabular}
\caption{NFE scaling under common argmax decoding in the final step. Each
entry is GenPPL / mean within-sequence unigram entropy over 256 generations.
The architecture, tokenizer, and sequence length are shared; the checkpoint is
the only model-level change across the horizontal rule. The dense base-Duo
64-NFE frontier in Table~\ref{tab:frontier-values-app} is an independent launch
rather than a pooled estimate.}
\label{tab:nfe-scaling-app}
\end{table}

\begin{table}[H]
\centering
\small
\setlength{\tabcolsep}{3.5pt}
\begin{tabular}{@{}lrrrr@{}}
\toprule
Checkpoint / NFE & Full support & top-$p$ 0.95 & top-$p$ 0.9 & CRS \\
\midrule
Duo-distilled / 8  & 0.84 / 0.006 & 0.79 / 0.005 & 0.31 / 0.008 & 0.11 / 0.005 \\
Duo-distilled / 16 & 0.43 / 0.001 & 0.26 / 0.004 & 0.18 / 0.006 & 0.26 / 0.004 \\
Duo-distilled / 32 & 0.64 / 0.007 & 0.59 / 0.010 & 0.51 / 0.009 & 0.18 / 0.008 \\
Duo-distilled / 64 & 0.84 / 0.015 & 0.39 / 0.014 & 0.19 / 0.009 & 0.53 / 0.015 \\
\midrule
Duo / 8  & 1.61 / 0.003 & 0.99 / 0.003 & 1.15 / 0.002 & 0.19 / 0.012 \\
Duo / 16 & 0.76 / 0.009 & 0.56 / 0.003 & 0.64 / 0.006 & 0.12 / 0.004 \\
Duo / 32 & 1.62 / 0.008 & 1.40 / 0.011 & 0.90 / 0.010 & 0.64 / 0.012 \\
Duo / 64 & 1.62 / 0.011 & 0.64 / 0.009 & 0.77 / 0.009 & 0.62 / 0.012 \\
\bottomrule
\end{tabular}
\caption{Four-shard standard errors for Table~\ref{tab:nfe-scaling-app}.
Each entry is GenPPL SE / unigram-entropy SE. Figure~\ref{fig:nfe-scaling}
shows point estimates without bars for legibility.}
\label{tab:nfe-scaling-se-app}
\end{table}

\paragraph{Paired endpoint audit.}
\mbox{\citet{deschenaux2026psi}} report the final
noise-removal step as an ancestral categorical draw, whereas our main comparison
uses coordinatewise argmax decoding in the final step. To isolate this protocol
difference, Table~\ref{tab:deschenaux-endpoint-audit} reuses each pre-terminal
reverse path and changes only that final output rule. The ancestral base-Duo cells recover the published 32-NFE
operating region at both tested support levels. The effect is not a uniform
quality bonus: argmax lowers entropy in every row, improves native/full-support
GenPPL, and worsens base-Duo GenPPL at $p=0.9$. Endpoint choice therefore
substantially accounts for the apparent numerical discrepancy, but its effect
depends on checkpoint and support restriction.

\begin{table}[H]
\centering
\small
\setlength{\tabcolsep}{3.8pt}
\begin{tabular}{@{}llrrr@{}}
\toprule
Checkpoint & Support & Reported ancestral & Paired ancestral & Paired argmax \\
\midrule
Duo & $p=1.0$ & $96.76\;(5.57)$ & $100.11\;(5.58)$ & $87.62\;(5.23)$ \\
Duo & $p=0.9$ & $44.24\;(5.40)$ & $44.78\;(5.40)$ & $53.57\;(5.22)$ \\
Duo-distilled & $p=1.0$ & $68.35\;(5.54)$ & $61.13\;(5.54)$ & $45.50\;(5.40)$ \\
Duo-distilled & $p=0.9$ & $35.92\;(5.41)$ & $34.90\;(5.42)$ & $33.54\;(5.35)$ \\
\bottomrule
\end{tabular}
\caption{32-NFE endpoint audit. Entries are GenPPL (within-sequence unigram
entropy, shown to the two-decimal precision available in the reported source).
The two paired columns share the pre-terminal path and differ only in
the last output rule; each uses four seed shards and 64 generations. Values in
the reported column are read from the ancestral-sampling tables of
\mbox{\citet{deschenaux2026psi}}. These paired 64-generation cells are
independent of the 256-generation NFE sweep in
Table~\ref{tab:nfe-scaling-app}, so their argmax point estimates need not
coincide.}
\label{tab:deschenaux-endpoint-audit}
\end{table}

\begin{table}[H]
\centering
\footnotesize
\begin{minipage}[t]{0.48\linewidth}
\centering
\textbf{Duo-distilled}\\[2pt]
\begin{tabular}{@{}lrr@{}}
\toprule
Control & GenPPL & Entropy \\
\midrule
$p=1.0$ & $41.79\pm0.84$ & $5.351\pm0.015$ \\
$p=0.95$ & $34.34\pm0.39$ & $5.326\pm0.014$ \\
$p=0.9$ & $29.76\pm0.19$ & $5.296\pm0.009$ \\
$p=0.85$ & $25.94\pm0.34$ & $5.251\pm0.013$ \\
$p=0.8$ & $23.34\pm0.39$ & $5.205\pm0.016$ \\
$p=0.75$ & $20.38\pm0.23$ & $5.148\pm0.017$ \\
\midrule
$D=128$ & $33.09\pm0.69$ & $5.336\pm0.014$ \\
$D=256$ & $30.54\pm0.62$ & $5.328\pm0.015$ \\
$D=512$ & $28.27\pm0.53$ & $5.317\pm0.015$ \\
$D=1024$ & $27.62\pm0.42$ & $5.312\pm0.014$ \\
\bottomrule
\end{tabular}
\end{minipage}
\hfill
\begin{minipage}[t]{0.48\linewidth}
\centering
\textbf{Duo}\\[2pt]
\begin{tabular}{@{}lrr@{}}
\toprule
Control & GenPPL & Entropy \\
\midrule
$p=1.0$ & $80.35\pm2.28$ & $5.233\pm0.011$ \\
$p=0.95$ & $59.53\pm0.64$ & $5.250\pm0.009$ \\
$p=0.9$ & $48.42\pm1.82$ & $5.221\pm0.014$ \\
$p=0.85$ & $39.86\pm0.73$ & $5.195\pm0.009$ \\
$p=0.8$ & $32.34\pm0.72$ & $5.114\pm0.009$ \\
$p=0.75$ & $28.33\pm0.64$ & $5.075\pm0.012$ \\
\midrule
$D=128$ & $51.96\pm0.78$ & $5.243\pm0.011$ \\
$D=256$ & $41.99\pm0.46$ & $5.216\pm0.011$ \\
$D=512$ & $35.75\pm0.62$ & $5.184\pm0.012$ \\
$D=1024$ & $34.60\pm0.53$ & $5.170\pm0.014$ \\
\bottomrule
\end{tabular}
\end{minipage}
\caption{Mean $\pm$ SE over four seed shards for the 64-NFE operating points.
All generations have length 1024. The base-Duo top-$p$ curve combines
protocol-aligned runs: $p=0.95$ comes from the fixed-procedure NFE sweep and
the remaining thresholds come from the denser 64-NFE sweep. Those denser reruns
give $80.35/48.42$ rather than $78.31/47.20$ GenPPL for $p=1.0/0.9$ in
Table~\ref{tab:nfe-scaling-app}; we do not pool the launches. The base-Duo
comparison is therefore descriptive.}
\label{tab:frontier-values-app}
\end{table}

The 8-NFE CRS cells commit $799.94$ positions on Duo-distilled and $795.50$ on
Duo on average. With a one-third warmup, both begin selection after two steps.
The frozen path-total budget is therefore aggressive relative to the available
low-NFE trajectory. These points diagnose one fixed procedure under reduced
compute rather than a tuned low-NFE frontier.

The base-Duo top-$p$ curve folds between $p=1.0$ and $p=0.95$ in unigram
entropy. We therefore do not invert that segment. Quantitative matched checks
use the monotone $p\le0.9$ branch and $D_{\max}\in\{256,512,1024\}$; the
$D_{\max}=128$ point is displayed only as a descriptive endpoint.

\subsection{How do the checkpoints differ in selection?}
\label{app:checkpoint-selector}

We evaluate both checkpoints on the same exact uniform-forward states and let
each checkpoint select 25 positions by Top-1 confidence with radius-four
exclusion. Both checkpoints are then evaluated on both selected sets. Across
$t\in\{0.7,0.67,0.6,0.5,0.4,0.2\}$, Duo-distilled self-selection has clean-label
accuracy $0.9$--$2.1$ percentage points higher than base-Duo self-selection. The selected-position Jaccard index
is only $0.43$--$0.53$. In contrast, when the selected positions are held fixed,
the two checkpoints agree on the argmax label at least $99.75\%$ of the time,
and their clean-label accuracy differs by at most $0.25$ points. Thus, in this
shared-state audit, the checkpoints differ mainly in which positions they
select. Once those positions are fixed, their argmax labels and clean-label
accuracies are nearly identical. Both selectors remain
overconfident: mean confidence is nearly one while selected-token error remains
$3.5$--$8.1\%$. This shared-state diagnostic does not include writeback or a
generated trajectory, so it does not attribute the full free-running endpoint
difference.

\subsection{Sampler-level warmup-by-duration ablation}
\label{app:warmup-duration}

We vary the full-support warmup fraction and the visibility horizon on
Duo-distilled at 64 NFE. The warmup fractions
$0.016,0.17,0.33,0.50$ correspond to $1,11,21,32$ warmup steps. Every
persistent cell uses $D_{\max}=512$. All cells retain full clean-token support
at every step and use argmax decoding in the final step. The cells share pair
seeds and RNG resets within each shard. Changing the requested warmup also
changes the two-piece reverse time grid, even when $H=0$ disables selection;
the $H=0$ rows therefore control for this schedule split rather than producing
bitwise-identical paths. Four paired seed shards contribute 256 generations
per cell.

\begin{table}[H]
\centering
\footnotesize
\setlength{\tabcolsep}{4.0pt}
\begin{tabular}{@{}crrrr@{}}
\toprule
Warmup steps & Native GenPPL & CRS GenPPL & Native entropy & CRS entropy \\
\midrule
1  & $42.15\pm0.98$ & $20.26\pm0.18$ & $5.359\pm0.014$ & $5.198\pm0.013$ \\
11 & $41.87\pm0.78$ & $25.93\pm0.27$ & $5.357\pm0.015$ & $5.297\pm0.013$ \\
21 & $41.79\pm0.84$ & $28.27\pm0.53$ & $5.351\pm0.015$ & $5.317\pm0.015$ \\
32 & $42.18\pm0.74$ & $29.83\pm0.47$ & $5.356\pm0.014$ & $5.328\pm0.014$ \\
\bottomrule
\end{tabular}
\caption{Endpoint means $\pm$ shard SE for the two policies plotted in
Figure~\ref{fig:warmup-permanent}. Native sampling keeps every position
revisable; \CRS{} keeps each selected label visible for all remaining
steps.}
\label{tab:warmup-permanent-endpoints-app}
\end{table}

\begin{table}[H]
\centering
\footnotesize
\setlength{\tabcolsep}{4.0pt}
\begin{tabular}{@{}crrrr@{}}
\toprule
Warmup steps & $\Delta$GenNLL $(8-0)$ & $\Delta$entropy $(8-0)$ &
$\Delta$GenNLL $(\infty-8)$ & $\Delta$entropy $(\infty-8)$ \\
\midrule
1  & $-0.486\pm0.023$ & $-0.1276\pm0.0082$ & $-0.246\pm0.004$ & $-0.0336\pm0.0022$ \\
11 & $-0.265\pm0.010$ & $-0.0486\pm0.0023$ & $-0.214\pm0.005$ & $-0.0123\pm0.0037$ \\
21 & $-0.211\pm0.008$ & $-0.0343\pm0.0020$ & $-0.180\pm0.003$ & $+0.0004\pm0.0013$ \\
32 & $-0.202\pm0.006$ & $-0.0335\pm0.0011$ & $-0.144\pm0.006$ & $+0.0052\pm0.0015$ \\
\bottomrule
\end{tabular}
\caption{Paired longer-horizon minus shorter-horizon contrasts (mean$\pm$SE
over four seed shards). Negative values mean lower evaluator GenNLL or lower
within-sequence unigram entropy under longer visibility.}
\label{tab:warmup-duration-app}
\end{table}

The $H=0$ endpoint is stable under the artificial phase split: GenPPL ranges
from $41.79$ to $42.18$, and unigram entropy ranges from $5.351$ to $5.359$.
This stability localizes the interaction to persistent exposure rather than
the time-grid split.
Delaying selection reduces both the likelihood gain and the entropy cost of
persistent context. At warmup fraction $0.33$, extending $H=8$ to $H=\infty$
leaves unigram entropy nearly unchanged, but the distinct 4-gram
fraction (Dist-4)~\citep{li2016diversity} decreases from $0.9405$ to $0.9359$, the repeated 4-gram
fraction (Rep-4) increases from
$0.0330$ to $0.0367$, and the across-sample similarity score
Self-BLEU~\citep{papineni2002bleu,zhu2018texygen}
increases from $0.2205$ to $0.2248$. The sequence-level movement
shows that unigram entropy understates the full duration cost. This
sampler-level ablation tests the qualitative warmup mechanism; quantitative
calibration of the oracle risk in Proposition~\ref{prop:warmup-blackwell} and
selection of an optimal warmup remain separate questions.

\subsection{Unconditional factorial}

We use `s-sahoo/duo-distilled' with the GPT-2 tokenizer. Every cell has 64 NFE,
length 1024, warmup fraction $0.33$, full-support warmup, canonical confidence
and spacing selection, $D_{\max}=512$, spacing radius four, float64 sampling
arithmetic to avoid the finite-precision categorical-sampling boundary effects
discussed by \citet{zheng2025secretly}, per-step count cap 25, and argmax decoding in the final step. Each of four seed shards
$8200,8210,8220,8230$ contributes 64 paired samples per cell, for 256 total. Initial states
and RNG streams are reset within each pair.
The released Duo top-$p$ implementation retains the first label beyond an
exact cumulative-mass equality; this floating-point boundary convention does
not affect the conversion identity in Section~\ref{sec:top-p-loo},
which conditions on the retained set actually used.

GenNLL is the mean next-token negative log likelihood under GPT-2-Large and
GenPPL is its exponential. Pythia-1.4B supplies the second evaluator. Entropy
is the mean, over generated sequences, of each sequence's empirical unigram
token entropy. It is distinct from an across-sequence distributional entropy.
For a multiset $\mathcal X$ of generated sequences and the multiset
$\mathcal G_n(x)$ of $n$-grams in $x$, our conventions are
\[
 \operatorname{Dist}\text{-}n
 =\frac{|\bigcup_{x\in\mathcal X}\operatorname{uniq}\mathcal G_n(x)|}
 {\sum_{x\in\mathcal X}|\mathcal G_n(x)|},
 \qquad
 \operatorname{Rep}\text{-}n
 =\frac1{|\mathcal X|}\sum_{x\in\mathcal X}
 \left(1-\frac{|\operatorname{uniq}\mathcal G_n(x)|}
 {|\mathcal G_n(x)|}\right).
\]
Self-BLEU-4 compares each
generation with the remaining generations. Higher entropy and Dist-4 indicate
broader lexical usage; lower Rep-4 and Self-BLEU indicate less repetition or
greater across-sample variation.
All entropy and NLL values use natural logarithms.
Because sentence-level BLEU depends on tokenization and
smoothing~\citep{post2018clarity,chen2014smoothing}, the unconditional table
uses leave-one-out GPT-2-tokenized Self-BLEU with equal weights over one- to
four-grams and add-one smoothing. The fixed-prefix analysis uses the same
weights with method-1 smoothing within each prompt group.

\begin{table}[H]
\centering
\small
\setlength{\tabcolsep}{4pt}
\begin{tabular}{@{}ccrrrrr@{}}
\toprule
Late $p$ & $H$ & GenPPL & Entropy & Dist-4 & Rep-4 & Self-BLEU \\
\midrule
1.0 & 0 & 41.79 & 5.351 & 0.9483 & 0.0272 & 0.2113 \\
1.0 & 8 & 33.84 & 5.317 & 0.9405 & 0.0330 & 0.2205 \\
1.0 & $\infty$ & 28.27 & 5.317 & 0.9359 & 0.0367 & 0.2248 \\
0.9 & 0 & 32.79 & 5.332 & 0.9432 & 0.0306 & 0.2207 \\
0.9 & 8 & 26.94 & 5.303 & 0.9346 & 0.0368 & 0.2294 \\
0.9 & $\infty$ & 26.40 & 5.302 & 0.9323 & 0.0389 & 0.2292 \\
0.85 & 0 & 30.58 & 5.321 & 0.9415 & 0.0316 & 0.2235 \\
0.85 & 8 & 25.66 & 5.293 & 0.9328 & 0.0388 & 0.2293 \\
0.85 & $\infty$ & 25.74 & 5.292 & 0.9328 & 0.0391 & 0.2285 \\
0.8 & 0 & 28.38 & 5.302 & 0.9388 & 0.0334 & 0.2275 \\
0.8 & 8 & 24.51 & 5.279 & 0.9308 & 0.0399 & 0.2331 \\
0.8 & $\infty$ & 24.89 & 5.283 & 0.9299 & 0.0414 & 0.2305 \\
\bottomrule
\end{tabular}
\caption{Complete unconditional top-$p$ by visibility-duration factorial. ``Late'' means
that the listed threshold is applied after full-support warmup. Means are over
four paired shards and 256 generations per cell.}
\label{tab:unconditional-full-app}
\end{table}

\subsection{Fixed-prefix transfer}

We use 1024 held-out OpenWebText~\citep{gokaslan2019openwebtext} prefixes, each 50 GPT-2 tokens long, with five
paired 50-token continuations per cell. The 50 prefix coordinates remain fixed
throughout sampling and are excluded from the selector; diversity summaries
use continuation tokens only. The sampler uses 64 NFE, the path-total
confidence-deficit budget $D_{\max}=25$, and argmax decoding in the final step. Reported conditional
perplexity (PPL) scores the continuation given the fixed prefix; conditional
negative log-likelihood (NLL)
is its log-scale counterpart. We evaluate 5120 generations per cell.

\begin{table}[H]
\centering
\small
\setlength{\tabcolsep}{3pt}
\begin{tabular}{@{}lcrrrrr@{}}
\toprule
Top-$p$ schedule & $H$ & GPT-2 & Pythia & Entropy & Dist-4 & Rep-4 \\
\midrule
Late 1.0 & 0 & 128.54 & 110.39 & 3.508 & 0.9606 & 0.0264 \\
Late 1.0 & 8 & 103.50 & 88.79 & 3.471 & 0.9480 & 0.0380 \\
Late 1.0 & $\infty$ & 83.96 & 71.75 & 3.453 & 0.9373 & 0.0469 \\
Late 0.9 & 8 & 83.21 & 71.62 & 3.461 & 0.9402 & 0.0446 \\
Late 0.8 & 0 & 88.43 & 76.71 & 3.488 & 0.9497 & 0.0356 \\
Late 0.8 & 8 & 74.77 & 64.55 & 3.452 & 0.9353 & 0.0482 \\
Late 0.8 & $\infty$ & 73.31 & 62.83 & 3.446 & 0.9341 & 0.0491 \\
\midrule
All 0.8 & 0 & 72.13 & 63.25 & 3.450 & 0.9393 & 0.0424 \\
\bottomrule
\end{tabular}
\caption{Fixed-prefix transfer point estimates. Entries labeled \emph{Late} use full-support
warmup and the listed $p$ afterward; \emph{All} uses that $p$ throughout.
GPT-2 and Pythia columns are conditional perplexity. Diversity summaries use
continuation tokens. Paired shard-level uncertainties for the primary
evaluator-NLL contrasts are reported below.}
\label{tab:conditional-full-app}
\end{table}

Under GPT-2-Large, the paired $H=0\to8$ NLL reductions are
$0.2165\pm0.0061$ at late $p=1.0$ and $0.1678\pm0.0053$ at late $p=0.8$;
the $H=8\to\infty$ reductions are $0.2091\pm0.0050$ and
$0.0198\pm0.0021$. The corresponding paired difference-in-differences are
$0.0487\pm0.0056$ and $0.1893\pm0.0069$ nats per token. Pythia-1.4B gives
interaction estimates $0.0451\pm0.0064$ and $0.1857\pm0.0093$. Positive
interaction values mean that lowering late $p$ attenuates the horizon effect.

For the paired evaluator-NLL contrasts against $p=0.8$ at every step with $H=0$, the
late-$p=0.8,H=8$ policy changes GPT-2/Pythia NLL by
$+0.0361\pm0.0105$/$+0.0208\pm0.0111$ nats per token; the $H=\infty$ policy
changes them by $+0.0163\pm0.0096$/$-0.0064\pm0.0100$. These are policy
contrasts on different generated trajectories and do not estimate the effect
of adding persistent context to the every-step schedule.
Within-prompt continuation Self-BLEU is $0.0378$ for every-step $p=0.8,H=0$,
$0.0351$ for late $p=0.8,H=8$, and $0.0357$ for late $p=0.8,H=\infty$.

\subsection{Support-by-duration interaction}
\label{app:support-duration}

Figure~\ref{fig:space-time} contains a complete unconditional factorial crossing
post-warmup top-$p\in\{1.0,0.9,0.85,0.8\}$ with visibility duration
$H\in\{0,8,\infty\}$. Temporary eight-step context remains useful throughout
that sweep. The marginal gain from extending $H=8$ to $H=\infty$
rapidly disappears as support restriction becomes stronger and becomes
slightly negative at late $p=0.8$. Fixed-prefix continuation has complete
duration contrasts only at $p\in\{1.0,0.8\}$. Its measured long-horizon gain
is also much smaller at $p=0.8$ than at $p=1.0$, although it remains positive;
these two endpoints do not constitute the full factorial.

\begin{figure}[H]
  \centering
  \includegraphics[width=\linewidth]{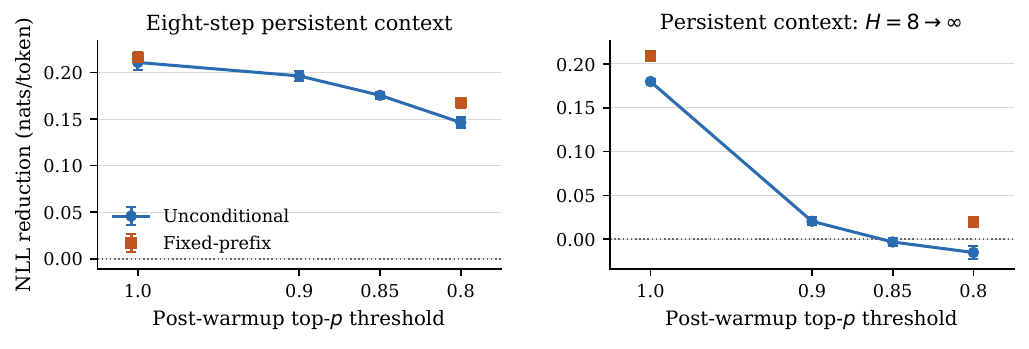}
  \caption{Interaction between post-warmup support restriction and context
  duration. The unconditional series is the complete $4\times3$ factorial.
  The left panel shows the gain from keeping selected context visible
  for eight steps. The right panel shows the additional gain from keeping it
  visible for all remaining steps. Blue lines are unconditional generation; orange
  squares are the two observed fixed-prefix endpoints at $p=1.0$ and $p=0.8$.
  Moving right means stronger
  restriction. Error bars are standard errors over four paired shards.}
  \label{fig:space-time}
\end{figure}

Corollary~\ref{cor:mode-diminishing-returns} motivates this qualitative sign
test. We do not identify its idealized support evidence $S$ or context evidence
$T$ with a checkpoint statistic. The factorial instead asks whether stronger
support restriction attenuates the measured marginal horizon gain. The
corollary assumes exact, mode-consistent evidence, whereas the experiment uses
a learned denoiser on its own free-running trajectory. The slight negative
unconditional endpoint therefore lies outside the theorem's guaranteed regime
rather than contradicting its exact-model statement.

\subsection{Earlier storage-output control}
\label{app:storage-output-control}

This earlier 64-NFE control uses a $p=0.9$ warmup and sweeps post-warmup
top-$p$. Final-only and CRS store the same selected labels and copy them to the
same returned coordinates; they differ in whether those labels are also
inserted into later denoiser inputs. Their paired difference isolates repeated
input visibility beyond output storage. The fully revisable arm instead uses
argmax decoding in the final step and is included as a reference, rather than
as an output-rule-matched third arm.

\begin{figure}[H]
  \centering
  \includegraphics[width=0.9\linewidth]{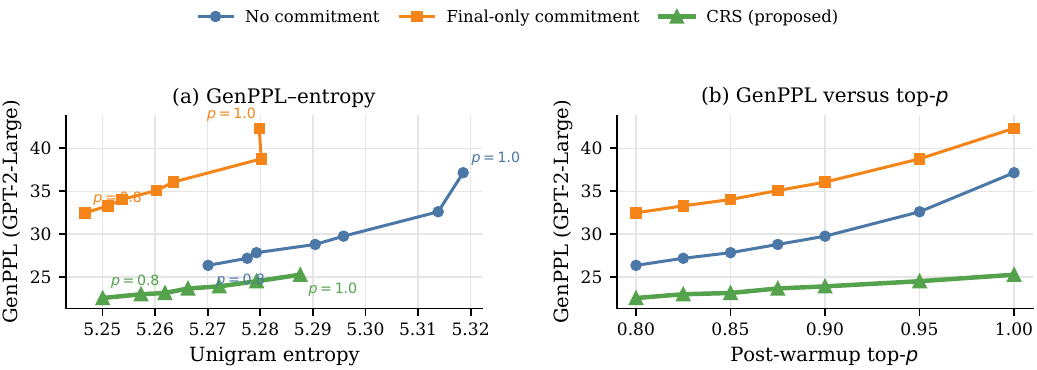}
  \caption{Earlier storage--feedback control. Panel (a) uses unigram entropy
  as the diversity coordinate; panel (b) shows the post-warmup top-$p$ sweep.
  Vertical error bars show GenPPL standard errors over four paired shards; at
  this scale, they are generally smaller than the plotting markers.}
  \label{fig:storage-feedback-app}
\end{figure}

\FloatBarrier

\section{Supporting Mechanism and Frontier Results}
\label{app:mechanism-tables}

\subsection{Scalar-temperature control}

For positive categorical laws $\nu$ and $\mu$, define
$\nu_\beta(a)=\nu(a)^\beta/\sum_b \nu(b)^\beta$. Then
$\mu=\nu_\beta$ for some $\beta\ge0$ exactly when
\begin{equation}
 \log\frac{\mu(a)}{\mu(b)}=\beta\log\frac{\nu(a)}{\nu(b)}
 \qquad\text{for every }a,b.
 \label{eq:temperature-characterization-app}
\end{equation}
The forward direction is immediate. For the reverse direction, fix one
reference label $b$, exponentiate the equality, and normalize. Scalar
temperature preserves rankings and ties for $\beta>0$; the boundary
$\beta=0$ instead maps every positive $\nu$ to the uniform law and collapses all
rankings to ties. Define the irreducible temperature residual by
\begin{equation}
 R_{\mathrm{temp}}(\mu;\nu)
 :=\min_{\beta\ge0}\KL(\mu\kldiv \nu_\beta).
 \label{eq:temperature-residual-app}
\end{equation}
The objective
\[
 f(\beta)=-\Ent(\mu)-\beta\E_\mu[\log \nu(A)]
 +\log\sum_a \nu(a)^\beta
\]
for $A\sim\mu$ is convex because
$f''(\beta)=\operatorname{Var}_{\nu_\beta}[\log \nu(A)]\ge0$. The following
same-state audit minimizes this objective separately at every coordinate, a
stronger control than fitting one global temperature. None of the primary
fitted optima occurs at $\beta=0$.

\subsection{Same-state temperature and token values}

The same-state audit follows full-support \CRS{} ($H=\infty$) paths at 64 NFE. In
each of four shards it uses eight paths, seven post-warmup steps, and eight
deterministically spaced active coordinates per step, for 1792 primary
coordinate--step pairs. At pair $j$, let $h_j^{\mathrm{ctx}}$ be the native
categorical denoiser output under the current input overwrite and let $h_j^0$ be the
output from the identical mutable state and time after removing all current
input overwrites. The audited coordinate itself is uncommitted and
has the same input token in both inputs. We fit
\begin{equation}
 \widehat\beta_j
 :=\arg\min_{\beta\in[0,4096]}
 \KL\!\left(h_j^{\mathrm{ctx}}\kldiv(h_j^0)_\beta\right),
 \qquad
 R_j:=\KL\!\left(h_j^{\mathrm{ctx}}
 \kldiv(h_j^0)_{\widehat\beta_j}\right).
 \label{eq:same-state-fit-app}
\end{equation}
For shard $s$, the reported residual fraction is the pooled ratio
\begin{equation}
 \mathcal R_s
 :=\frac{\sum_{j\in\mathcal J_s}R_j}
 {\sum_{j\in\mathcal J_s}
 \KL(h_j^{\mathrm{ctx}}\kldiv h_j^0)}.
 \label{eq:same-state-pooled-residual-app}
\end{equation}
Float64 log probabilities and 48 bisection steps solve the convex fit; no
primary optimum reaches either boundary; in particular, none uses
$\beta=0$, the uniform-reference limit.
The mean and SE of $\mathcal R_s$ over the four shards are
$79.4\pm2.1\%$. The top-ranked token differs between
$h_j^{\mathrm{ctx}}$ and $h_j^0$ on $30.2\pm0.8\%$ of coordinates, and
$99.89\pm0.06\%$ contain a pairwise rank inversion in the union of their two
top-64 sets (log-probability tie tolerance $10^{-12}$). The initial-step
empty-context control has zero top-rank changes and zero rank inversions.

At full clean-token support, the value intervention holds the selected positions,
transition streams, and argmax decoding in the final step fixed. Argmax values give
$28.27\pm0.53$ GenPPL. Replacing them by the current second-ranked token gives
$77.75\pm1.80$, and sampling a top-$p=0.9$ alternative after excluding the
argmax gives $78.36\pm1.56$. The refreshed-value branch, which keeps positions
visible but updates values from each current argmax, gives
$26.42\pm0.39$. Token identity matters, while immutable identity is not
required for hard visible context to alter the path. Refreshed replays the
cumulative selected-position masks of the immutable-value CRS trajectory and
replaces the exposed labels after each step for the following model input.
Table~\ref{tab:refreshed-matched-app} compares the two value rules at matched
diversity.

\subsection{Hidden-coordinate support}

The later-support audit uses four shards, eight \CRS{} ($H=\infty$) paths per
shard, and up to eight spaced commitment events per path. At offsets $1,2,4,$
and $8$, it evaluates the stored label under the overwritten input, the same
\CRS{} trajectory with only the audited coordinate restored to its mutable
noisy state, and a paired fully revisable trajectory. The restored-coordinate
view is an input intervention, not exact marginalization over that coordinate.

For the restored-coordinate view, stored-label probabilities at offsets
$1,2,4,8$ are $0.7483\pm0.0060$, $0.8323\pm0.0044$, $0.8230\pm0.0052$, and
$0.8291\pm0.0009$; paired revisable values are $0.4939\pm0.0262$,
$0.5542\pm0.0359$, $0.5481\pm0.0422$, and $0.5630\pm0.0430$. Restoration changes the
coordinate in $144,114,108,96$ of $256,224,224,223$ events. Means and standard
errors use four shards; direct recomputation from the shard files gives
$0.00089$ for the last restored-coordinate SE. Conditional on the restoration
actually changing that coordinate, the restored probabilities are
$0.7300\pm0.0129$, $0.8402\pm0.0209$, $0.8199\pm0.0297$, and
$0.8458\pm0.0230$; the paired revisable probabilities are
$0.4534\pm0.0327$, $0.5295\pm0.0493$, $0.5191\pm0.0486$, and
$0.5378\pm0.0580$. Because other visible context and noisy-state changes
remain, the audit supports a trajectory-mediated rather than isolated account.

\subsection{Matched GenPPL--entropy comparisons}

On Duo-distilled, the full-support \CRS{} ($H=\infty$) sweep varies the confidence budget and
compares each point with a global top-$p$ curve using the same argmax decoding in the final step,
$p\in\{0.75,0.8,0.85,0.9,0.95,1.0\}$. For each shard we interpolate that
shard's observed global curve at that shard's observed CRS diversity and form
the residual before taking the four-shard mean and SE; no extrapolation is
used. Both sweeps were generated in the same launch with aligned pair seeds,
initial states, and RNG resets across cells, so the four interpolated values
are paired shard contrasts. The entropy-matched GenNLL residual retains its
sign across the budget sweep: for
$D_{\max}=128,256,512,1024$, respectively, it is
$-0.089\pm0.029$, $-0.130\pm0.038$, $-0.154\pm0.031$, and
$-0.155\pm0.030$ nats per token.
At $D_{\max}=512$, matching separately on each diversity statistic gives the following table,
whose entries are CRS-minus-global-top-$p$ GenNLL residuals in nats per
token; each column names the diversity statistic used for matching. Negative
values favor CRS.
\begin{table}[H]
\centering
\small
\begin{tabular}{@{}lrrrr@{}}
\toprule
Value rule & Entropy & Dist-4 & Rep-4 & Self-BLEU \\
\midrule
Immutable CRS & $-0.154\pm0.031$ & $-0.025\pm0.048$ & $+0.070\pm0.036$ & $-0.144\pm0.019$ \\
Refreshed & $-0.095\pm0.024$ & $-0.012\pm0.027$ & $+0.076\pm0.021$ & $-0.163\pm0.019$ \\
Refreshed $-$ immutable & $+0.059\pm0.013$ & $+0.013\pm0.022$ & $+0.006\pm0.017$ & $-0.019\pm0.010$ \\
\bottomrule
\end{tabular}
\caption{Immutable and refreshed visible labels against the same global
top-$p$ reference, evaluated by GPT-2-Large. The first two rows are
entropy-, Dist-4-, Rep-4-, or Self-BLEU-matched GenNLL residuals; the last row
is their paired difference. Entries are means $\pm$ SE over seeds
$8200,8210,8220,8230$, with 64 sequences per seed and policy. All policies use
64 NFE, length 1024, and final argmax decoding; the context policies use
$w=0.33$, $D_{\max}=512$, and no scheduled position expiry. Refreshed uses the
same cumulative selected-position masks as immutable CRS. Matching sorts
the six observed global top-$p$ points by each diversity statistic and
interpolates GenNLL within each shard. All targets are bracketed in all four
shards. Negative residuals favor the context policy over global top-$p$;
a positive paired difference favors immutable values.}
\label{tab:refreshed-matched-app}
\end{table}

The Refreshed outputs and their immutable and Native controls were generated
in the original token-value cohort. Their Native and immutable sequences
match the corresponding frontier cohort exactly, including pair seeds and
RNG fingerprints, so the stored global sweep supplies the reference without
additional generation. Refreshed lowers raw GenPPL while also lowering
unigram entropy and increasing Rep-4; at matched entropy, retaining the
original labels gives the more favorable residual. Both value rules favor
context at matched Self-BLEU, whereas Rep-4 matching favors global top-$p$.

\clearpage\section{Finite retention and SOL}
\label{app:ibis}
This section reports additional finite-retention and distributional comparisons. The original workshop experiments and the Refreshed
matched comparison above retain their own cohorts and protocols. The
finite-retention configuration below is not the Refreshed policy and does not
replace the full-support comparison.

\subsection{Input context and finite retention}
\label{app:ibis-protocol}
The model input is $\bm u_t=\bm x_t[C_t\leftarrow\bm c]$, where $C_t$ is the
set of currently visible retained labels. The reverse transition conditions
on the mutable state $\bm x_t$, including at selected coordinates.
After warmup, previously unselected positions are ranked by raw confidence
$\max_a h_{\theta,t}^i(a\mid\bm u_t)$ and selected subject to a confidence-deficit
budget and within-batch spacing. Each position is selected once.
For $N_{\rm tr}=63$ transitions and $W=\operatorname{round}(0.10N_{\rm tr})=6$,
there are $R=57$ opportunities. With $D_{\max}=512$, the per-opportunity
budget is $512/57$ and the count cap is $\lceil1024/57\rceil=18$.
Spacing requires distance greater than four within a new batch; unused
budget does not carry over. Selected labels enter the next input and remain
visible for at most $H=8$ calls. Optional raw-disagreement release removes a
label when the current argmax, computed from the input with the label
visible, differs from it; the
release takes effect after the current transition. Expiry and release do
not restore an earlier sampled state.

\begin{figure}[htbp]\centering
\includegraphics[width=\linewidth]{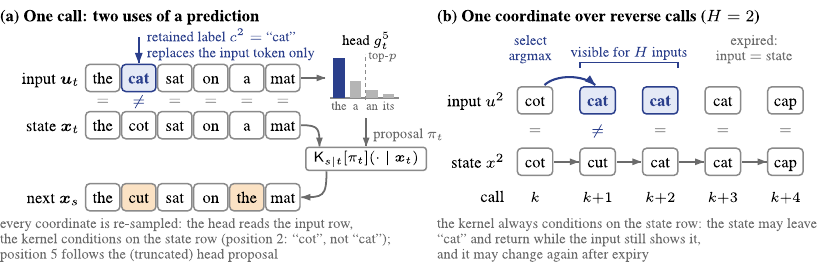}
\caption{Visible input and sampled state. A selected label remains in later model inputs while the state is re-sampled. The $H=2$ timeline illustrates the convention; the additional comparison uses $H=8$. Here $K$ denotes the reverse update in Equation~\eqref{eq:learned-reverse-main}.}
\label{fig:ibis-input-state}
\end{figure}

Global top-$p$ truncates throughout; late top-$p$ starts after warmup.
The common budget is 64 model evaluations: 63 transitions followed by a
final coordinatewise argmax call. We distinguish the last sampled state
$X=\bm x_\varepsilon$ from the final argmax output $Y$, evaluated with the
remaining context visible. The last selection can affect $Y$ but not the
already sampled $X$. These are two readouts of the same trajectory.
All curves use 256 texts in 16 paired shards of 16, with 1,024 generation
tokens, on each of Duo and Duo-distilled. The label cohort and SOL analysis
use the same saved outputs.

For the same finite-retention cohort, matching diversity remains important.
Table~\ref{tab:ibis-entropy-rep4} reports raw-label residuals against late top-$p$
at matched unigram entropy and Rep-4. The final argmax output has negative
entropy-matched point estimates, while the Rep-4 residuals are positive.
This retains the distinction between token-frequency diversity and repetition.

\begin{table}[H]\centering\small
\begin{tabular}{llrr}\toprule
Checkpoint/output & Release & Entropy matched & Rep-4 matched\\\midrule
Duo $X$ & off & $+0.036$ & $+0.199$\\
Duo $X$ & on & $-0.049$ & $+0.145$\\
Duo $Y$ & off & $-0.053$ & $+0.118$\\
Duo $Y$ & on & $-0.105$ & $+0.122$\\
Duo-distilled $X$ & off & $-0.075$ & $+0.152$\\
Duo-distilled $X$ & on & $-0.086$ & $+0.128$\\
Duo-distilled $Y$ & off & $-0.124$ & $+0.118$\\
Duo-distilled $Y$ & on & $-0.119$ & $+0.105$\\
\bottomrule\end{tabular}
\caption{Raw-label GenNLL residuals to the same-warmup late top-$p$ curve, in nats per token. Point estimates, not confidence bounds. Each matching statistic is interpolated within each of the 16 paired shards, then residuals are averaged. Negative favors context.}\label{tab:ibis-entropy-rep4}
\end{table}

Reference thresholds are $p\in\{0.3,0.4,0.5,0.6,0.7,0.8,0.9,0.95,1\}$.
Matching first uses an exact vertex and otherwise the first adjacent pair in
increasing threshold order that brackets the policy's statistic, without
extrapolation. GenNLL pools evaluator losses within each shard and then
averages shards equally. These point estimates are separate from the SOL
jackknife intervals below and the original workshop's four-shard SEs.

\subsection{Distributional comparison with SOL}
SOL \citep{kornhardt2026sol} compares generated texts with held-out reference
texts in encoder feature space; lower values mean closer distributions in
that representation. It is evaluated on the entire corpus, not averaged
from per-shard SOL values. The reference comprises 256 OpenWebText validation
texts from the official SOL reference, selected by a fixed random permutation
with seed 20261010; three further disjoint groups of 256 texts give real-real
controls. No outcome-based reference selection is used.

The computation uses the official SOL implementation,
1,024 token-space directions, 1,024 projections, 64 quantile bins and
lengthscale 0.1, with token-direction and primary projection seeds zero.
GPT-2-Large encodes the unchanged saved strings with its tokenizer and a
1,024-token horizon, excluding padding. The encoder scale is 1,280.
No entropy matching is used.

On the global sweep, moving from moderate to strong truncation lowers GenPPL
while worsening SOL. For example, on Duo-distilled $Y$, changing $p=0.8$ to
$p=0.3$ lowers GenPPL from 22.979 to 9.283, while GPT-2-Large SOL rises from
1.4374 to 2.8798.
The SOL curves turn upward under strong truncation;
the sampled-grid minima depend on the readout. Moderate truncation
can improve both metrics, so the comparison does not imply that all
truncation is harmful.

At the same late $p=0.8$, context with or without release lowers GenPPL and
SOL for $Y$ on both checkpoints, at lower entropy. In the displayed GPT-2-Large panel these SOL point estimates are also below
the global and late sweep points. This is a descriptive grid comparison, not a simultaneous
best-of-grid guarantee. Table~\ref{tab:ibis-sol} reports the same-$p$ contrasts.
Without release, the saved-state intervals contain zero for both checkpoints.

\begin{table}[H]\centering\small
\begin{tabular}{llc}\toprule
Checkpoint/output & Release & SOL difference (GPT-2-Large)\\\midrule
Duo $X$ & off & $-0.0199\;[-0.0507,+0.0108]$\\
Duo $X$ & on & $-0.1494\;[-0.1851,-0.1138]$\\
Duo $Y$ & off & $-0.2246\;[-0.2731,-0.1762]$\\
Duo $Y$ & on & $-0.3343\;[-0.3790,-0.2897]$\\
Duo-distilled $X$ & off & $-0.0121\;[-0.0371,+0.0129]$\\
Duo-distilled $X$ & on & $-0.0347\;[-0.0688,-0.0007]$\\
Duo-distilled $Y$ & off & $-0.0773\;[-0.1103,-0.0443]$\\
Duo-distilled $Y$ & on & $-0.0991\;[-0.1366,-0.0616]$\\
\bottomrule\end{tabular}
\caption{SOL differences, raw-label context minus same-$p$ late top-$p$, $p=0.8$. Approximate pointwise paired delete-one-group jackknife 95\% intervals, conditional on encoder and projections; no entropy matching. Negative favors context.}\label{tab:ibis-sol}
\end{table}

Intervals are approximate pointwise paired delete-one-group jackknife
$t_{15}$ intervals: each deletion removes the same generation shard and
preassigned 16-document reference block across policies. They condition on
the encoder and projections and are not multiplicity-adjusted confirmatory
intervals. Projection seeds 0, 1 and 2, sample sizes 64, 128 and 256, and
three disjoint real-real controls were retained as sensitivity checks.
Finite-sample real-real distances need not be zero. The complete saved
plans and compact numerical summaries accompany the source. These comparisons are conditional on the fixed GPT-2-Large representation.
SOL is neither a human preference score nor a task success measure.

\end{document}